\documentclass{article}
\PassOptionsToPackage{hypertexnames=false}{hyperref}
\usepackage{arxiv,times}
\usepackage{custom}

\usepackage{placeins}
\usepackage[mode=tex]{standalone}
\usepackage{comment}

\title{Careful Judge: Safe and Efficient Human--AI Collaborative Decision Making}

\iclrfinalcopy
\author{Chenyu Zhang\thanks{Work done while at an internship at NVIDIA.}\\
MIT\\
\texttt{zcysxy@mit.edu} \\
\And Rachel Luo  \\
NVIDIA \\
\texttt{raluo@nvidia.com}\\
\And Boyi Li  \\
NVIDIA \\
\texttt{boyil@nvidia.com}\\
\And Anjali Parashar\footnotemark[1]\\
MIT\\
\texttt{anjalip@mit.edu}
\And Marco Pavone \\ 
NVIDIA, Stanford University \\
\texttt{mpavone@nvidia.com}
\And Apoorva Sharma \\
NVIDIA \\
\texttt{apoorvas@nvidia.com}
}

\begin{document}

\maketitle

\begin{abstract}
	In human--AI collaborative decision making, human review can prevent unsafe AI decisions, but each human judgment is costly.
	Treating human intervention as a one-off fallback misses the opportunity to improve future AI decisions, yet adaptively learning from human feedback changes the model and can invalidate previously calibrated safety guardrails.
	We approach this challenge with \method---\underline{c}alibrated \underline{a}daptive \underline{r}ectification and \underline{e}scalation---an end-to-end pipeline that combines AI models and human reviewers to guarantee safe, human-aligned decisions, while continuously learning from human feedback to achieve greater automation with fewer human queries.
	\method is principled, general, modular, and works with any black-box AI model.
	We equip it with a rectification module that improves AI decisions by learning the structure of human--AI misalignment.
	Our novel adaptive calibration module guarantees risk control at every time step for any rectification module.
	Experiments on four safety-critical real-world datasets spanning language, vision, driving, and robotics demonstrate that \method achieves human-aligned decisions while reducing human queries by 25--81\% relative to baselines.
\end{abstract}

\section{Introduction}

Human--AI collaboration is increasingly used in safety-critical applications, including
data triage and evaluation \citep{hemmer2025complementarity,mozannar2020consistent},
autonomous driving \citep{luo2025Sim2ValLeveraging,sinha2024RateInformedDiscovery},
robot monitoring \citep{sermanet2025generating},
and model alignment \citep{christiano2017deep,lee2023rlaif,bai2022constitutional}.
In these settings, AI predictions represent an inexpensive, but potentially misaligned, way to scale reliable, but costly, human judgment. 
Deciding when to trust AI and when to call for human review creates a safety--efficiency trade-off \citep{el-yaniv2010FoundationsNoisefree}.

Existing methods address different parts of this trade-off.
Some provide safety guarantees but may require frequent human intervention
\citep{jung2025Trustescalate,ren2023RobotsThat};
others learn from human feedback to improve efficiency without providing such guarantees
\citep{angelopoulos2023PredictionPoweredInference,christiano2017deep}.
Naively combining the two directions may amplify their respective weaknesses and worsen the trade-off%
---splitting already scarce human feedback between model improvement and safety checks could make the former noisier and the latter more conservative, as is evident in \cref{sec:exp}.
Still, we ask for the best of both worlds: \emph{how can a human--AI system continue to make safe decisions while learning from human feedback to reduce future intervention?}

A positive answer requires true collaboration: a human provides feedback when AI is uncertain, and AI becomes more confident when facing similar scenarios in the future.
Such dynamic interaction should uncover the structure of human--AI misalignment and close this gap, making AI decisions safer and human intervention less frequent.
However, this structure may be difficult to learn when human review is sparse and selective.
Moreover, safety guarantees must keep pace with each model update.
Specifically, queries selectively reviewed by a human can differ statistically from future ones, and earlier safety checks may no longer be valid after model updates, so calibration must efficiently track shifts in both the data distribution and the model.

These challenges call for model-agnostic risk control for every decision to be made,
and model-aware adaptive learning after every decision is made.
Model awareness enables learning \emph{efficiency}, while model agnosticism ensures \emph{safety} in the presence of model misspecification.
Driven by these two objectives, we propose \emph{calibrated adaptive rectification and escalation} (\method), a principled, general, and modular framework that combines AI models and human reviewers into an ensemble decision-maker. \method guarantees safe, human-aligned decisions and continuously learns from human feedback to achieve greater automation while reducing human queries.

We instantiate \method in binary classification, using an exact risk decomposition (\cref{sec:decomp}) to motivate its four modules:
\begin{itemize*}
	\item[(\cref{sec:rect})] \emph{Residual Gauss} rectifies AI predictions and uncertainty estimates by
    fitting a Gaussian process to the residual between human and AI decisions;
	\item[(\cref{sec:cali})] \emph{Confident Jack} calibrates the rectified model with statistical and computational efficiency;
	\item[(\cref{sec:esc})] \emph{Exacting Gates} escalates queries to a human using a scenario-specific threshold informed by calibration; and
	\item[(\cref{sec:ada})] \emph{Active Max} adaptively selects informative human queries.
\end{itemize*}
All modules are designed for flexible offline--online learning schedules and share two objectives: safety and efficiency.
An overview of \method is provided in \cref{fig:pipe}.

\begin{figure}[t]
	\centering
	\begin{subfigure}[c]{0.40\textwidth}
		\centering
        \vspace{-1.5em}
		\includegraphics[width=\linewidth]{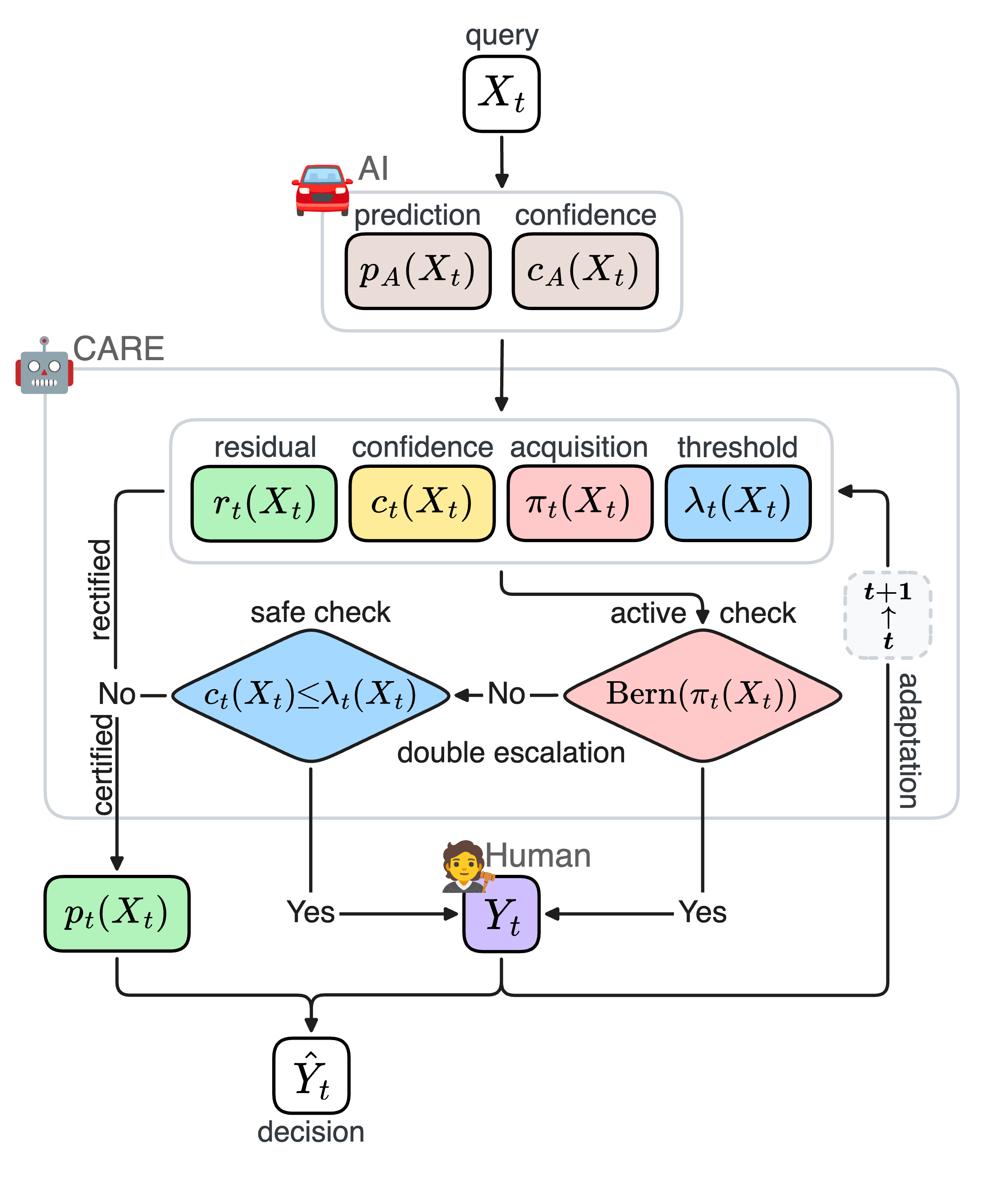}
        \vspace{-2.5em}
		\label{fig:pipe-loop}
	\end{subfigure}
    \unskip\hfill{\color{lightgray}\vrule}\hfill
	\begin{subfigure}[c]{0.59\textwidth}
		\centering
		\includegraphics[width=\linewidth]{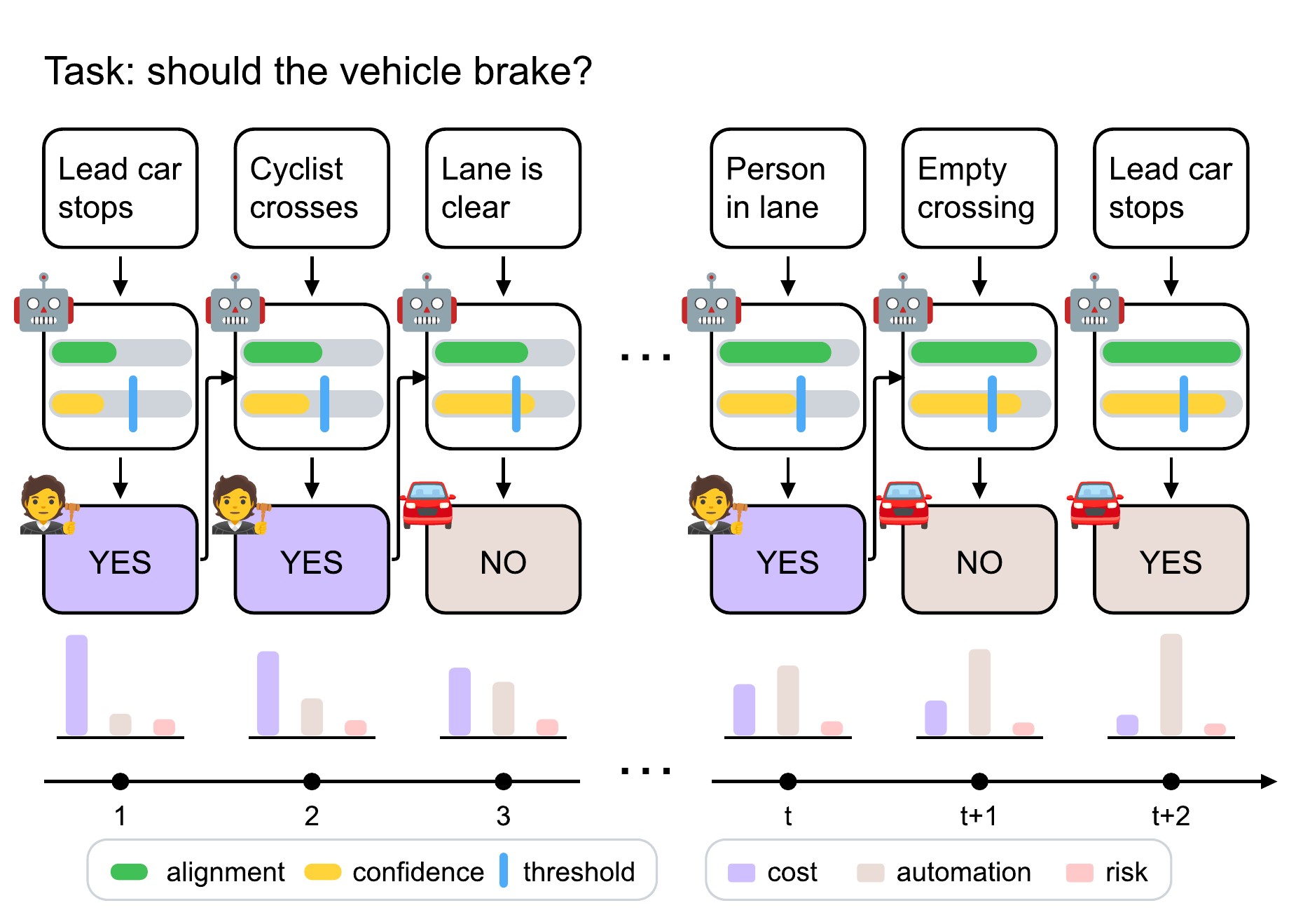}
		\label{fig:pipe-life}
	\end{subfigure}
	\caption{\textbf{(left)} Per-sample decision pipeline. Upon a query, \method rectifies AI's prediction and confidence by correcting the human--AI residual and reducing uncertainty, based on past human feedback. The rectified quantities are then passed through two checks: an active acquisition rule that escalates informative queries, and a confidence gate that ensures the risk of the final human--AI decision is below the tolerance. Once a query is escalated, the human's feedback is used as the final decision and to update \method's modules.
		\textbf{(right)} Adaptive learning from human feedback. \method learns to adapt to the revealed human preference over time and across scenarios, continuously increasing alignment and reducing the uncertainty of the rectified AI model, and thus improving efficiency.}
	\label{fig:pipe}
\end{figure}

Specifically, our contributions are:
\begin{itemize}
	\item We formalize human--AI collaboration as a general statistical decision problem with a flexible offline--online schedule, covering different scenarios and model lifecycles.
	\item We derive an exact conditional risk decomposition into four sources, each motivating a distinct component of \method---an end-to-end human--AI collaboration pipeline.
	\item We develop a correlation-aware Bayesian residual rectifier that combines a learned kernel with AI-imputed uncertainty. 		It propagates human feedback across similar scenarios, enabling faster adaptation and reducing future escalation frequency.
	\item We develop adaptive weighted confidence-based cross-conformal calibration that reuses labels collected for learning and accommodates arbitrary model updates under history-dependent acquisition. It provides model-agnostic marginal risk control at every time step.
	\item We evaluate \method on four safety-critical real-world datasets spanning language, vision, driving, and robotics, demonstrating human-aligned decisions with substantially fewer human queries.
\end{itemize}

\subsection{Related work}

\textbf{AI-accelerated human feedback.}
As AI models are applied to increasingly complex tasks, high-quality supervision is a key bottleneck. To alleviate this, some works consider using learned models to scale supervision. Reinforcement learning from human feedback (RLHF) distills human preference data into a reward model \citep{christiano2017deep}, while other techniques bring in human feedback at a higher level, in rubric or constitution design \citep{bai2022constitutional, lee2023rlaif, zheng2023judging}. Such techniques have been instrumental in advancing open-world physical AI, with LLM/VLM rewards scoring decision-making in robotics and autonomous vehicles \citep{ma2023eureka, sima2023drivelm, huang2024vlm-rl, sermanet2025generating, rocamonde2024vision}. However, ensuring alignment between the AI feedback and what a human expert would provide remains critical, especially in safety-critical settings. Recent literature applies ideas from selective classification and abstention \citep{el-yaniv2010FoundationsNoisefree,vardhan2026ImprovingSelective,vishwakarmaPromisesPitfallsThresholdbased2024} to applications of AI feedback, allowing AI to escalate hard cases to human experts \citep{jung2025Trustescalate}. We extend these ideas to allow for online adaptation, improving efficiency while maintaining alignment guarantees. Moreover, we find that the local online adaptation of \method can mitigate the tendency of static selective escalation to exacerbate disparities among regions in the input space \citep{jones2021selective}.

\textbf{Model-assisted inference.}\label{rmk:ppi}
A related line of research explores combining cheap proxy measurements and limited expensive measurements in the context of variance-reduced statistical inference, where the cheap proxy can be a prediction model \citep{angelopoulos2023PredictionPoweredInference,zrnic2026ActiveStatistical} or a simulator \citep{luo2025Sim2ValLeveraging,luo2026X4ValLearning, badithela2025suresim}.
\citet{angelopoulos2023PredictionPoweredInference} propose prediction-powered inference (PPI) that uses empirical risk minimization to correct the model's systematic bias. While its fitted rectifier can capture some local structure, we model potential residual correlations explicitly and update each prediction's confidence through its posterior. PPI instead provides a global confidence bound for its bias correction. Greater model capacity can reduce approximation error but may increase estimation variability, motivating prior refitting as human feedback accumulates.

\textbf{Conformal calibration.}\label{rmk:cali}
Sharing the goal of model-agnostic risk control, cross-conformal prediction \citep{vovk2012CrossconformalPredictors} and Jackknife+ \citep{barber2021PredictiveInference} improve data efficiency over split conformal prediction \citep{papadopoulos2002InductiveConfidence,lei2018DistributionFreePredictive,jung2025Trustescalate}, which splits the available human feedback into disjoint sets for training and calibration.
To correct known or tractable covariate shifts, \citet{tibshirani2020ConformalPrediction} introduce \emph{weighted exchangeability}, and \citet{prinster2022JAWSAuditing} apply it to obtain weighted Jackknife+.
These methods do not directly handle the history-dependent covariate shifts induced by active acquisition in our setting, especially during the online period, when the weight calculation is intractable.
Online model adaptation adds another layer of distribution shift.
Existing online conformal methods \citep{gibbs2021AdaptiveConformal,gibbs2022ConformalInference,bhatnagar2023ImprovedOnline,feldman2023AchievingRisk} often treat distribution shifts as adversarial and focus on time-averaged regret or coverage.
Our goal is risk control at every time step without conservative fallbacks: calibration should track model improvement and reduce escalation as the model learns.

\subsection{Problem setup} \label{sec:setup}

We instantiate human--AI collaborative decision making in binary classification.
Consider an unlabeled dataset $(X_{t})_{t\ge 1}$ from the feature space $\X$, with corresponding human oracle labels $(Y_{t})_{t\ge 1}$ from the label space $\Y=\{0,1\}$, revealed upon query escalation; we write $Z_{t}=(X_{t},Y_{t})$ for the full data.
We assume the dataset is drawn from an i.i.d. stream and the human oracle provides correct labels.
The inexpensive AI model provides a probabilistic classifier with an attached \emph{confidence} score: $(p_A, c_{A}): \X \mapsto [0,1]^{2}$, which is not assumed to be accurate or calibrated.

Given these resources, we consider an \emph{end-to-end} task: a framework must use them for all required stages, including but not limited to training, validation, calibration, and deployment. Any labels needed at these stages must be obtained through human queries.
We consider two metrics: \emph{risk} and \emph{cost}.
For a loss function $L$, e.g., zero--one loss $L_{01}: \mathcal{Y}^{2}\to \R_{\ge 0}, (\hat y,y)\mapsto\Ind\{\hat y\neq y\}$ and Brier score $L_{2}: [0,1]^{2} \to \R_{\ge 0}, (\hat{p},y)\mapsto (\hat{p}-y)^{2}$, and a classifier $f$, the risk is $R_{L}(f)=\EE[L(f(X),Y)]$.
Our default risk is the zero--one risk, i.e., the misclassification probability $R_{01}(f)=P(f(X)\neq Y)$.
Cost is measured by the human query complexity, i.e., the end-to-end total number of human queries.
We say a classifier is \emph{safe} if its risk is below a user-specified tolerance $\epsilon\in(0,1)$.
We say an algorithm is more (cost-)\emph{efficient} than another if it has lower human query complexity.
Distinguish this from computational efficiency.
Our goal is to maximize efficiency subject to the safety constraint.

\section{Careful Judge: A framework for risk-controlled adaptive AI-assisted decision making} \label{sec:decomp}

We motivate our method by examining the risk of an AI decision.

\begin{proposition}[Risk decomposition; informal]\label{prop:decomp}
	The AI model's Brier risk has an exact decomposition:
	\[
		R_{2,x}(p_{A}) =& \EE_{Y}[(p_A(X)-Y)^2\given X=x] \\
		=&  \underset{\substack{\ds\updownarrow\\ \ds\text{escalation}}}{\text{aleatoric ambiguity}}
		+ \underset{\substack{\ds\updownarrow\\ \ds\text{calibration}}}{\text{excess risk}}
		+ \underset{\substack{\ds \updownarrow\\ \ds\text{adaptation}}}{\text{epistemic uncertainty}}
		+ \underset{\substack{\ds\updownarrow\\ \ds\text{rectification}}}{\text{systematic bias}}
		.\]
\end{proposition}

The formal statement and proof are given in \cref{apx:decomp}.
\Cref{prop:decomp} identifies four sources of risk to be addressed to align the AI model with human preferences and deploy it safely:
\begin{itemize}
	\item Aleatoric ambiguity encodes the intrinsic hardness of the query, which is the irreducible risk of the Bayes-optimal human-aligned predictor. 	      When this risk exceeds the tolerance, we \emph{escalate} the query to the human oracle to reveal the correct zero-loss decision.
	\item Excess risk measures the gap between a learned predictor and the Bayes-optimal predictor due to finite data and model misspecification. We \emph{calibrate} the escalation rule to control the ensemble's risk despite imperfect learning.
	\item Epistemic uncertainty captures training-sample-specific information, including the correlation structure revealed by the queried data.
				We \emph{adapt} to this structure and prioritize informative queries that can improve future decisions.
	\item Systematic bias is the population-level misalignment of the AI model	      . We \emph{rectify} AI predictions by learning to correct this bias from human feedback.
\end{itemize}

This motivates calibrated adaptive rectification and escalation (\method), a pipeline that reuses every queried human label to adaptively rectify AI's systematic bias and epistemic uncertainty, calibrate the rectified AI model, and escalate to human when the rectified AI decision is still risky. This pipeline improves efficiency while controlling the risk of the resulting ensemble.
Formally, human queries inform a rectifier $r$ in the \emph{score} space (the range of a strictly increasing link function $g\colon [0,1]\to \bar{\R}$), an uncertainty-reduced confidence measure $c$, a calibrated escalation threshold $\lambda$, and an active acquisition function $\pi$. At time step $t$, these quantities assemble into a
probabilistic classifier:
\[\label{eq:care}
	\begin{cases}
		p_t({X_{t}}) \coloneqq g^{-1}(g(p_{A}(X_{t})) + {r_{t}}(X_{t})), & \text{if } {c_{t}}(X_{t}) > {\lambda _t}(X_t) \text{ and } 0=A_{t} \sim \operatorname{Bern}(\pi_{t}(X_{t})); \\
		\mathbf{1}\{ Y_{t}=1 \},                  & \text{otherwise.}
	\end{cases}
\]
In words, the rectified AI decision is accepted only when its confidence exceeds the calibrated threshold and the input is not selected for active acquisition. Otherwise, the human oracle supplies the final decision, and the queried label becomes available for subsequent learning and calibration. \Cref{fig:pipe} illustrates this decision flow.

\method is modular, as one can plug in any rectifier $r$, confidence measure $c$, escalation threshold $\lambda$, and acquisition function $\pi$ into the pipeline.
Nonetheless, in the next four sections, we develop these four modules jointly to pursue:

\begin{itemize}
	\item \emph{Safety.} Adaptive rectification continuously improves AI predictions, while adaptive calibration and escalation provide model-agnostic risk control as the model evolves.
	\item \emph{Efficiency.} We reuse each human query to update rectifier, quantify uncertainty, and calibrate the escalation threshold, maximizing query efficiency.
		Active acquisition and adaptive learning aim to reduce future escalation frequency.
\end{itemize}

\begin{remark}
	\Cref{prop:decomp} offers a triage of gaps left by existing tools.
	Classical selective classification \citep{el-yaniv2010FoundationsNoisefree} and calibrated threshold-based escalation \citep{jung2025Trustescalate} do not learn to improve the base model, so a biased or uncertain model may escalate every query under a small risk tolerance, leading to huge cost.
	Population-level rectifier regression \citep{angelopoulos2023PredictionPoweredInference} targets systematic bias, but its uncertainty quantification does not adapt to locally revealed correlation structure.
	Finally, active acquisition \citep{zrnic2026ActiveStatistical} can target epistemically informative points, but does not by itself learn a rectifier or an acquisition function, nor does it certify the risk of accepted decisions.
	\method closes all four gaps to achieve safety efficiently.
\end{remark}

Furthermore, we design \method to support flexible offline--online model lifecycles. At offline checkpoints, \method revisits accumulated human feedback in an order-invariant manner, refits rectification and rebuilds calibration, prioritizing query efficiency by making fuller use of queried labels. Between checkpoints, \method processes samples sequentially and updates rectification and calibration online.

\section{Residual Gauss: Locality-aware Bayesian rectification} \label{sec:rect}

To close reducible gaps in \cref{prop:decomp}, we use queries to learn the human--AI residual and rectify AI's decision to align with human preference.
AI's uncertainty about human preference can be correlated across the input space.
For example, an autonomous driving system that misjudges a pedestrian's intention at one intersection may make similar errors at similar intersections \citep{sinha2024RateInformedDiscovery}.
These correlations allow feedback on one query to improve decisions on related inputs, improving alignment---and thereby reducing escalation frequency---over time.
On the other hand, treating residuals as independent can lead to a query-all degradation when the initial uncertainty level is above the risk tolerance.

Therefore, \cref{prop:decomp} suggests a learning model to capture both the global systematic bias and local epistemic uncertainty of the human--AI residual.
We model this with a Gaussian process (GP) \citep{rasmussen2008GaussianProcesses}.
On the score scale of the link $g$, we place a $\mathrm{GP}(m,k)$ prior on the residual $r(x)=g(p_H(x))-g(p_{A}(x))$, where $p_H$ is the Bayes-optimal human-aligned predictor, and use human feedback to obtain the Bayesian rectifier $r_t$.
The mean $m$ captures global systematic bias, while the kernel $k$ captures correlations that let feedback transfer across inputs.

\textbf{Initialization: parametric prior regression.}
AI's raw confidence $c_{A}$ can offer useful information about uncertainty, even when it is not calibrated.
We incorporate it into a learnable GP prior:
\[
	r\sim\GP(m_\beta, k_{\tau,\rho,\phi}),
	\quad\text{with}\quad
	m_\beta(x)=x^T\beta,\quad
	k_{\tau,\rho,\phi}(x,x') = \tau^{2} (\rho k_{\phi}(x,x') + (1-\rho) k_A(x,x'))
	,\]
where the linear mean captures systematic bias, parameterized kernel $k_\phi$ captures residual correlations, and $k_A$ encodes the AI's uncertainty, such as an ensemble's empirical covariance or a diagonal kernel based on $c_A$.%
\footnote{We normalize $k_\phi$ and $k_A$ to have comparable scale. For simplicity, we use the same feature space $\X$ for the bias functional and both kernels.}
When $\tau\to 0$, our model reduces to the parametric linear regression model in \citet{angelopoulos2023PredictionPoweredInference}.
We refit the prior parameters at offline checkpoints, using, e.g., maximum likelihood estimation and expectation maximization.
This adaptive prior lets human feedback determine how much to trust the AI's confidence and how much additional residual structure to learn.

\textbf{Offline: nonparametric posterior inference.}
Given the queried labels, the GP posterior applies corrections at each input.
At time step $t$, let $R_t$ follow the (approximate) posterior distribution of the residual. The rectified posterior predictive probability $P_t = g^{-1}(g(p_A) + R_t)$ gives a rectified point prediction $p_t$ as its posterior mean.
We also rectify the confidence using the posterior risk estimate:
\[
	1-c_t = \mathbb{E}_{Y \!\sim\operatorname{Bern}(P_t)}[L(p_t, Y)]
	= \mathbb{E}_{P_t}[P_t L(p_t,1) + (1-P_t) L(p_t,0)]
	= p_t L(p_t,1) + (1-p_t) L(p_t,0)
	.\]
Thus, in binary classification, the same posterior mean provides both the rectified prediction and confidence measure.
For example, $c_t = \max\{p_t, 1-p_t\}$ for the 0--1 loss and $c_t = 1 - p_t(1-p_t)$ for the Brier loss.
Since binary labels give a non-Gaussian posterior, we approximate it using expectation propagation (EP) or a Laplace approximation \citep{rasmussen2008GaussianProcesses}.

\textbf{Online: assumed-density filtering.} \label{sec:adf.}
Between offline checkpoints, we keep the prior fixed and update the residual posterior after each instance of actively acquired human feedback using assumed-density filtering (ADF), the online counterpart of EP.
ADF obtains the Gaussian approximation of $R_{t+1}$ by matching its first two moments with those of the conditional posterior tilted by the new observation.
Each update requires only a one-dimensional integral and a rank-one covariance correction; the probit link gives closed-form updates.
An online maximum a posteriori (MAP) update provides an alternative online counterpart to the Laplace approximation. See \cref{apx:rect} for more details.

When the GP model is well-specified and the human--AI residual is small, in the sense of the $L^{2}(P_{X})$ norm, our rectified inference enjoys the variance reduction standard in model-assisted inference \citep{angelopoulos2023PredictionPoweredInference,zrnic2026ActiveStatistical,luo2025Sim2ValLeveraging}.
This formally characterizes the efficiency gain from using an aligned AI model.

\section{Confident Jack: Calibrating adaptive escalation and learning} \label{sec:cali}

Rectification alone does not guarantee safety; we need an adaptive safety check that keeps pace with rectifier learning.
We develop \emph{Confident Jack} to control the risk of trusting the rectified model at any fixed time, regardless of the underlying AI model or rectifier.
It uses the same queried labels used for rectifier learning for calibration, while accounting for how those labels were adaptively selected.

Confident Jack is a \emph{co-evolving weight-score particle system} $\{ (w_t^{(k)}, s_t^{(k)}  ) \}_{k=1}^{K}$ that attends to the test input $X_t=x$, giving a prediction set of the following form for any risk tolerance $\epsilon\in(0,1)$:
\[\label{eq:pred-set}
	\textstyle\calC_{t,\epsilon}(x)
	=\left\{y\in\Y:\sum_{k=1}^{K}w_t^{(k)}(x,y)\cdot \Ind\left\{s_t^{(k)} (z^{(k)})<s_t^{(k)} (x,y)\right\}<1-\epsilon/2\right\}
	,\]
where $\{ z^{(k)} \}$ are past human queries to be specified together with the particle system.
Intuitively, scores help check if a hypothetical label $y$ is consistent with past queries, which may have different relevance to the current test point due to selective escalation and adaptive model updates, and thus are weighted accordingly.
The prediction set retains labels that pass this weighted consistency check.
This prediction set is translated into a confidence threshold for escalation in \cref{sec:esc}.

\textbf{Initialization: confidence-based nonconformity score.}
To incorporate the model $\M_t$'s own uncertainty estimation, we define the score using its rectified confidence:
\[
	s(x,y;\M_t) \coloneqq c_t(x)\Ind\{f_t(x)\neq y\},
	\label{eq:nonconformity}
\]
where $f_t(x) = \Ind\{p_t(x)\ge 1 /2\}$ is the rectified model prediction.
The score is formally derived as the generalized inverse of the decision set of \cref{eq:care} in \cref{apx:cali}.
Intuitively, when $y$ is the true label, a confident mistake receives a higher nonconformity score than an uncertain one.
With $c_t(x)=\max\{p_t(x),1-p_t(x)\}$ in binary classification, this recovers the prediction-based score in \citet{romano2020ClassificationValid}.
Our construction is tailored to our escalation rule and accommodates model-provided confidence measures.

\textbf{Offline: weighted cross-conformal calibration.}
A score is associated with a model.
Prioritizing query efficiency, Confident Jack constructs these scores by fitting leave-one-fold-out models $\{ \M_{t}^{(k)} \}_{k=1}^{K}$.
Each queried label can then support learning in some folds and calibration in another, avoiding a permanently separate calibration set. %
Let $s^{(k)}_t \coloneqq s(\cdot ; \M_t^{(k)})$ and let $z^{(k)}$ be a designated point in the $k$th left-out dataset.
With leave-one-out folds, this recovers leave-one-out cross-conformal prediction \citep{romano2020ClassificationValid} and Jackknife+ \citep{barber2021PredictiveInference} for binary classification.

The weights restore the symmetry lost through adaptation (history-dependent selective escalation and model improvement).
For a candidate $(x,y)$, a weight answers how likely each held-out query would be to occupy the test slot when we do not know the order of the offline dataset plus the test point.
We find this probability by \emph{replaying} the logged acquisition policy under possible permutations of the held-out groups and the test point.
Let $G$ be the allowed permutations and $\Lambda_t^\sigma(x,y)$ the likelihood of the observed acquisition decisions under permutation $\sigma\in G$; the required restrictions on $G$ and the replay are detailed in \cref{apx:cali}.
Then, the exact conditional probability that the designated point $z^{(k)}$ occupies the test slot is
\[
	w_{t}^{(k)}(x,y) \propto \textstyle\sum_{\{\sigma\in G:\sigma(k)=K+1\}}\Lambda_t^\sigma(x,y) \qquad k=1,\dots,K,
	\label{eq:off-weight}
\]
where $\sigma(k)=K+1$ places the $k$th held-out group in the test slot $K+1$.
With a fixed, time-homogeneous query propensity $\pi_t(x)\equiv\pi(x)$, these become the usual normalized inverse-propensity weights $w_{t}^{(k)}(x)\propto1/\pi(x^{(k)})$ \citep{tibshirani2020ConformalPrediction,prinster2022JAWSAuditing}.

\textbf{Online: co-evolving calibration particles.}
We shift the focus to computational efficiency in the online period.
Replaying every possible history in \cref{eq:off-weight} would be costly online.
We thus restrict the replay to the offline period while fixing the observed order of the subsequent online stream.
Since the online period is order-invariant with respect to the offline dataset, we can factor the weight as:
\[
	w_{t}^{(k)}(x,y) \propto& w_{t_0}^{(k)} (x,y) \cdot \Lambda_{t_0:t}^{(k)}
	\label{eq:online-weight}
	,\]
where $t_0$ is the last offline checkpoint and $\Lambda_{t_0:t}^{(k)}$ is the likelihood of the observed online acquisition decisions for the $k$th held-out model; the derivation is deferred to \cref{apx:cali}.
This factorization places the main computational burden at the offline checkpoint to obtain $\{ w^{(k)}_{t_0} \}$.
During the online period, the particle system $\{ \M_{t}^{(k)} \}$ co-evolves with the main model $\M_t$ using the same online stream, giving adaptive scores as in \cref{eq:nonconformity} and efficiently maintaining the weights by appending the online likelihood as in \cref{eq:online-weight}.
A trade-off is that new online feedback improves the models immediately but joins the pool of held-out calibration points only at the next offline checkpoint.
Calibration therefore continues to use the offline feedback while its scores and weights adapt to the current model.

The co-evolution avoids treating model adaptation as adversarial drift.
For example, as the model particles improve, nonconformity scores on past feedback become lower, making it easier for \cref{eq:pred-set} to reject inconsistent test labels, yielding less conservative prediction sets and thus fewer escalations.

\begin{theorem}[Adaptive risk control]\label{thm:cali}
	Let $\calC_{t,\epsilon}$ be the prediction set defined in \cref{eq:pred-set} with the co-evolving particle system $\{ (w_t^{(k)}, s_t^{(k)}  ) \}_{k=1}^{K}$.
	Then, for any time $t$ (either an offline checkpoint or an online step) and any user-specified risk tolerance $\epsilon\in(0,1)$, we have
	\[
		P(Y_{t}\notin \calC_{t,\epsilon}(X_{t}))\le \epsilon.
	\]
\end{theorem}

\section{Exacting Gates: Safety escalation threshold on confidence} \label{sec:esc}

To get fast and simple escalation decisions, we translate the prediction set \cref{eq:pred-set} into confidence gates that adapt in both time and space, as used in \cref{eq:care}.
We have two observations:
\begin{enumerate*}[label=(\roman*)]
	\item stabilized learning closes the gap between the main model and the leave-one-out particles, and
	\item the exclusion of the opposite label $1-f_t(x)$ from \cref{eq:pred-set} indicates that it's \emph{safe} to accept the model's chosen label $f_t(x)$, although it does not guarantee the inclusion of $f_t(x)$ (proved in \cref{apx:esc}).
\end{enumerate*}
Therefore, to test if the model is confident enough to reject its opposite label, the threshold is set to
\[\label{eq:lambda}
	\textstyle\lambda_{t}(x) \coloneqq \q_{\frac{1-\epsilon /2}{W_{t}(x)}}\left(\sum_{k\in \mathcal{W}_{t}(x)} \frac{w_{t}^{(k)} (x)}{W_{t}(x)}\, \delta\{s_{t}^{(k)}(z^{(k)}) \}\right) + \gamma ^{c}_{t}(x)
	,\]
where weights $w^{(k)}_t(x)\coloneqq w^{(k)}_t(x,1-f_t(x))$ concern the opposite label,
$\delta\{x\}$ is the Dirac measure at $x$,
$W_t(x) = \sum_{k\in \mathcal{W}_{t}(x)} w_{t}^{(k)}(x)$ is the mass of the agreeing set $\mathcal{W}_{t}(x)\coloneqq \{ k\in [K]: f_{t}^{(k)} (x) = f_{t} (x) \}$,
and $\gamma ^{c}_{t}(x) = \max_{k\in \mathcal{W}_{t}(x)} \max\{0, c_{t}(x) - c_{t}^{(k)} (x)\}$ inflates the quantile to account for instability.
We use the convention that $\max$ over empty $\mathcal{W}_t(x)$ is zero and $\q_{c}\!=\!\infty$ for $c\!>\!1$.

Threshold \cref{eq:lambda} requires confidence to exceed a disagreement-inflated $(1-\epsilon /2)$-quantile of the agreeing particles' nonconformity scores.
Intuitively, if most particles agree with the current prediction and are mostly correct, then $c_t(x)$ passes the threshold more easily; if few particles agree with the current prediction and those that agree are overconfident, then $c_t(x)$ must be very high to pass the threshold.
Not only is $\lambda_t$ adaptive in time, but it also varies across the input space. For example, it's lower where the particle predictions are more stable (smaller $\gamma^{c}_t$ and larger $W_t$).

\begin{corollary}[Ensemble risk control]\label{cor:cali}
	Let $\hat{Y}_{t}$ be the decision output by \method \cref{eq:care} with the safety escalation threshold in \cref{eq:lambda}.
	Then, for any time $t$ (either an offline checkpoint or an online step) and any user-specified risk tolerance $\epsilon\in(0,1)$, we have $R_{01}(\hat{Y}_t) = P(\hat{Y}_{t}\neq Y_{t})\le \epsilon$.
\end{corollary}

The proof is deferred to \cref{apx:esc}, where we also relax the weight and score calculations to allow for approximation errors and accommodate computational constraints, offline or online.

\section{Active Max: Active acquisition for long-term efficiency} \label{sec:ada}

An active acquisition module $\pi$ in \cref{eq:care} independent of safety escalation is motivated by two needs.
First, calibration requires replaying the acquisition policy on counterfactual histories to compute the weights in \cref{eq:off-weight}.
If this acquisition policy coincides with the safety escalation policy, which is only known after calibration, then it creates a nested dependency and potentially exponential computational complexity.
Second, active acquisition can potentially improve long-term data efficiency by prioritizing the human queries that can reduce future escalation frequency, which cannot be achieved by the safety escalation rule alone.
This shares the same spirit as active statistical inference \citep{zrnic2026ActiveStatistical}, but with a different objective: instead of maximizing inference precision, we aim to maximize efficiency.

This double escalation mechanism separates the input and output of the calibration particle system and lets human review protect current decisions and guide learning for future ones.
Specifically, between offline checkpoints, actively escalated queries ($1=A_t \sim \operatorname{Bern}(\pi_t(X_t))$) update the rectifier and calibration particles; queries escalated only by the safety threshold ($1=\Ind{c_t(X_t)\le \lambda_t(X_t)}$) resolve the current decision, while their use for learning waits until the next checkpoint.
We refer the reader to \cref{apx:ada} for our active acquisition design.

\section{Experiments} \label{sec:exp}

We evaluate our method across four real-world safety-critical tasks, domains, and data modalities:
collision detection in autonomous driving on the trajectory dataset nuPlan \citep{caesar2022NuPlanclosedloop},
hate speech detection in LLM alignment on the text dataset HateBench \citep{shen2025HateBenchBenchmarking},
semantic safety in robotic tool use on the image dataset ASIMOV \citep{gemini2026roboticsSafety},
and hazard detection in the workplace on the video dataset iSafetyBench \citep{abdullah2025iSafetyBenchvideolanguage}.
Representative queries are shown in \cref{fig:dataset-samples}.
In all tasks, ``AI'' is an inexpensive model and ``human'' is a high-fidelity oracle.
Results are reported over ten random seeds; detailed setups and full results are in \cref{apx:exp}.

\begin{figure}[th]
	\centering
	\begin{subfigure}[t]{0.24\textwidth}
		\centering
		\includegraphics[width=\linewidth]{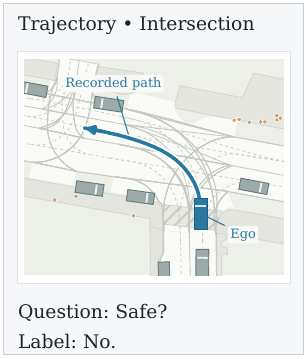}
		\caption{nuPlan}
		\label{fig:sample-nuplan}
	\end{subfigure}\hspace{0.013\textwidth}%
	\begin{subfigure}[t]{0.24\textwidth}
		\centering
		\includegraphics[width=\linewidth]{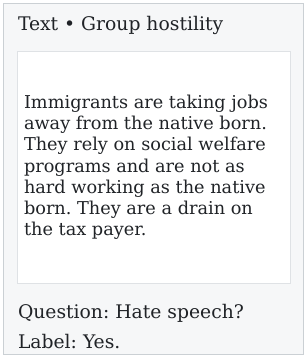}
		\caption{HateBench}
		\label{fig:sample-hatebench}
	\end{subfigure}\hspace{0.013\textwidth}%
	\begin{subfigure}[t]{0.24\textwidth}
		\centering
		\includegraphics[width=\linewidth]{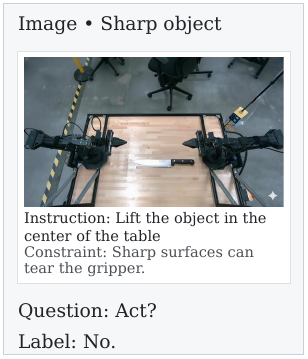}
		\caption{ASIMOV}
		\label{fig:sample-asimov}
	\end{subfigure}\hspace{0.013\textwidth}%
	\begin{subfigure}[t]{0.24\textwidth}
		\centering
		\includegraphics[width=\linewidth]{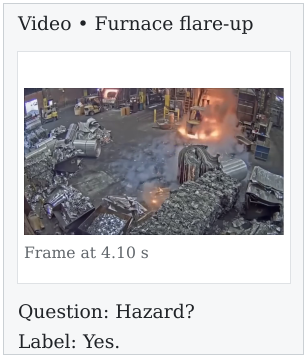}
		\caption{iSafetyBench}
		\label{fig:sample-isafety}
	\end{subfigure}
	\caption{Sample queries and benchmark labels.
		nuPlan shows recorded motion; its label refers to the benchmark simulation result.
		iSafetyBench shows one video frame.}
	\label{fig:dataset-samples}
	\vspace{-0.5\baselineskip}
\end{figure}

We compare with three baselines:
prediction-powered inference (PPI) \citep{angelopoulos2023PredictionPoweredInference}, which learns a linear rectifier without calibration;
selective escalation (SE) \citep{jung2025Trustescalate}, which calibrates AI confidence using conformal prediction without rectification;
and PPI+SE, which combines the two and uses separate data for rectification and calibration.
We also include three ablations: GPC+SE replaces the PPI rectifier in PPI+SE with a GP classification rectifier;
static \method shares offline labels between rectification and calibration without online adaptation;
and \method with PPI replaces our GP rectifier with the PPI rectifier.
The full \method uses GP rectification with an offline--online learning and calibration schedule.

Here we present some representative results demonstrating four properties that are observed consistently across all experiments.

\textbf{Cost efficiency under risk control.}
\method achieves the lowest query complexity while keeping empirical risk below each tested tolerance.
On iSafetyBench (\cref{fig:isafety-risk-sweep}), at the 5\% risk tolerance, \method escalates 21.7\% of queries to humans, compared with 57.1\% for SE, giving a 62\% reduction.

\begin{figure}[t]
	\centering
	\includegraphics[width=0.8\linewidth]{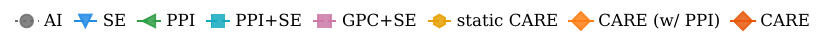}
    \hfill\\
	\begin{subfigure}[t]{0.24\textwidth}
		\centering
		\includegraphics[width=\linewidth]{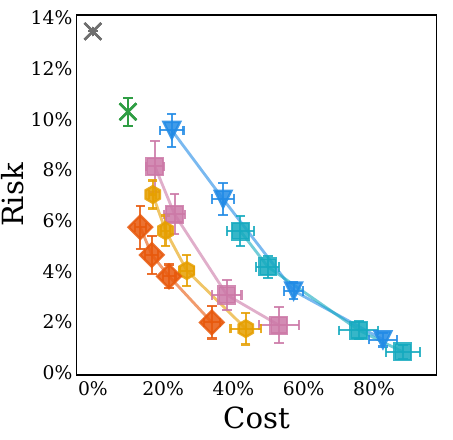}
		\caption{Risk--cost.}
		\label{fig:isafety-risk-sweep}
	\end{subfigure}\hfill
	\begin{subfigure}[t]{0.245\textwidth}
		\centering
		\includegraphics[width=\linewidth]{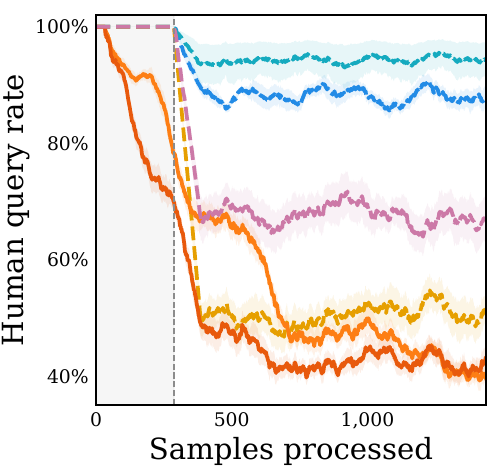}
		\caption{Human query rate.}
		\label{fig:asimov-query-rate}
	\end{subfigure}\hfill
	\begin{subfigure}[t]{0.245\textwidth}
		\centering
		\includegraphics[width=\linewidth]{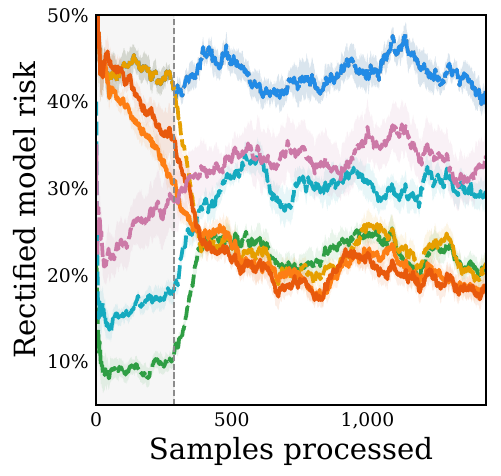}
		\caption{Rectified-model risk.}
		\label{fig:asimov-model-risk}
	\end{subfigure}\hfill
	\begin{subfigure}[t]{0.245\textwidth}
		\centering
		\includegraphics[width=\linewidth]{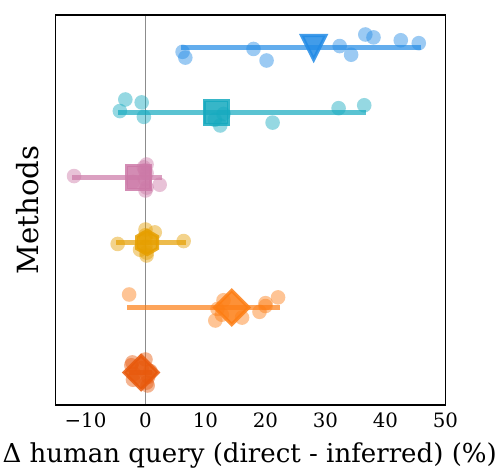}
		\caption{Confidence sensitivity.}
		\label{fig:isafety-confidence}
	\end{subfigure}
	\caption{		\textbf{(a)} iSafetyBench with AI-provided confidence and 10\% preparation data for static methods.
		The $x$-axis is the total human query rate, and the $y$-axis is the empirical mean risk.
		Lines connect nominal risk targets of 2.5\%, 5\%, 7.5\%, and 10\%. Crosses mark risk exceeding the corresponding target.
		A Pareto frontier would lower-bound all methods from the lower left.
		\textbf{(b)--(c)} ASIMOV with inferred confidence, 5\% risk target, and 20\% preparation data for static methods.
		The $x$-axis covers all samples regardless of escalation, and the $y$-axis uses a rolling window of 100 samples.
		\textbf{(d)} iSafetyBench query-rate differences at the 2.5\% risk target and 10\% static preparation. Positive deviations indicate efficiency degradation.
        }
	\label{fig:main-risk-cost}
	\vspace{-0.5\baselineskip}
\end{figure}

\textbf{Fast learning.}\label{sec:exp-asimov.}
\method quickly learns useful residual structure with fewer human queries.
On ASIMOV (\cref{fig:asimov-query-rate,fig:asimov-model-risk}), \method variants achieve significantly higher efficiency through early-stage learning.
Residual Gauss enables faster adaptation than PPI, and Confident Jack overcomes the data inefficiency exhibited by PPI+SE.

\textbf{Local adaptation.} \label{sec:exp-nuplan.}
While calibration controls marginal risk (\cref{thm:cali}), our GP model is designed to adapt locally, powering \method to achieve more balanced performance across the input space.
We examine this behavior across six physically defined regions in nuPlan (\cref{fig:nuplan-regions}).
At the final snapshot, \method achieves the largest model-risk reduction in every region, especially R1 and R2, where the original AI struggles.
Its mean ensemble risk stays below the target in all six regions and is closest to the target among methods below it in R5 and R6.

\begin{figure}[t]
	\centering
	\begin{subfigure}[t]{0.681553\textwidth}
		\vspace{0pt}
		\includegraphics[width=\linewidth]{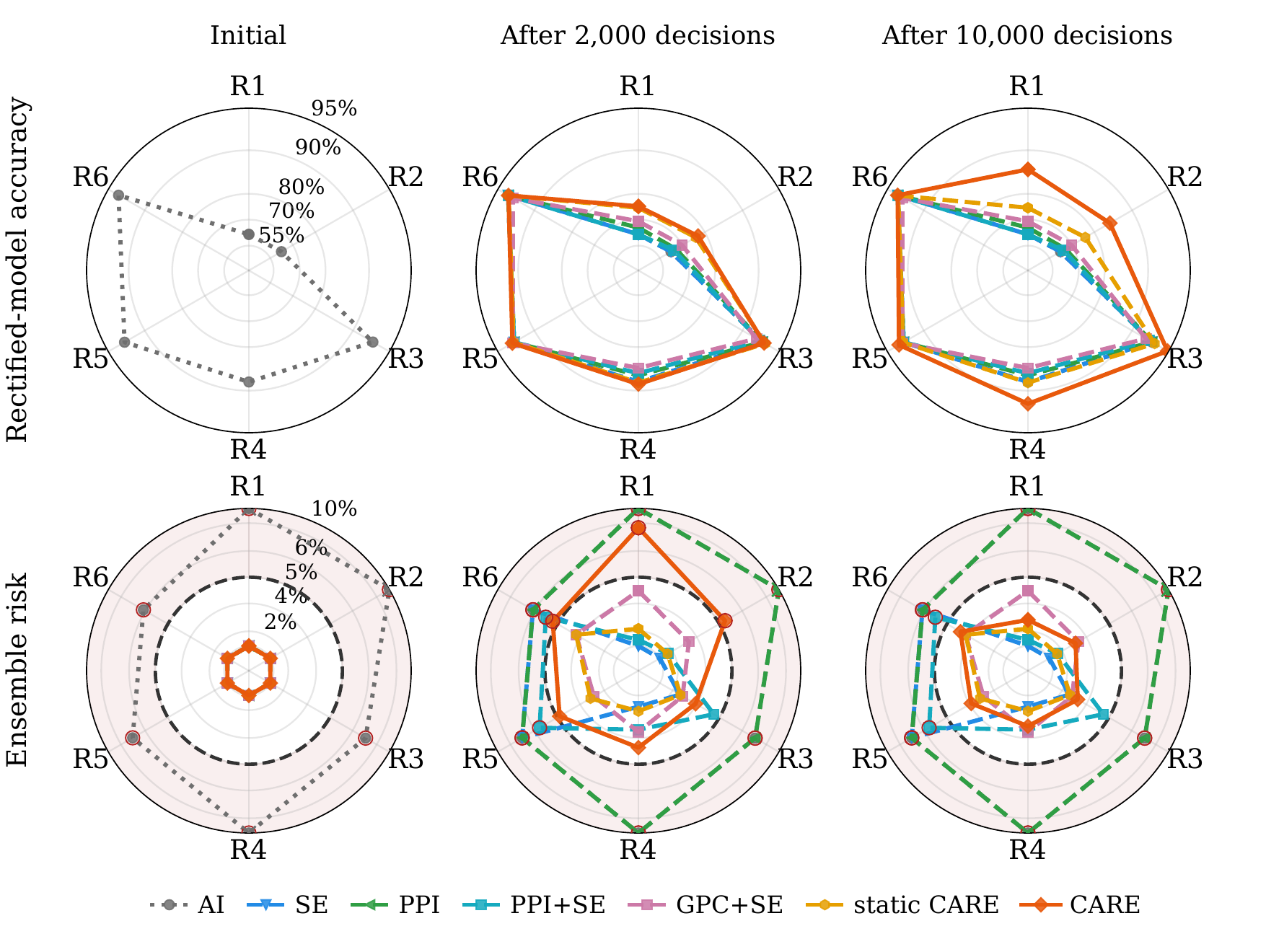}
		\vspace{-0.055\linewidth}
		\caption{Regional adaptation.}
		\label{fig:nuplan-regional-radars}
	\end{subfigure}\hspace{-0.02\textwidth}
	\begin{subfigure}[t]{0.33\textwidth}
		\vspace{0pt}
		\vspace{0.0\linewidth}
		\includegraphics[width=\linewidth]{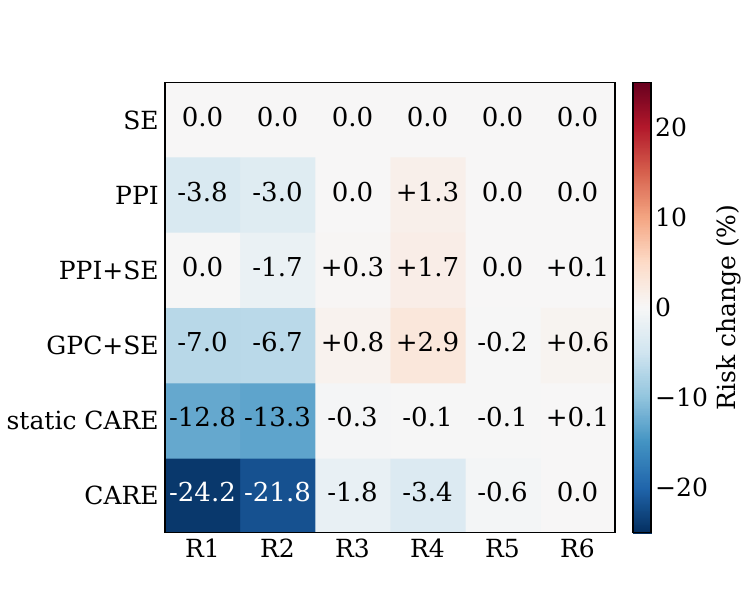}\par\nointerlineskip
		\vspace{-0.085\linewidth}
		\includegraphics[width=\linewidth]{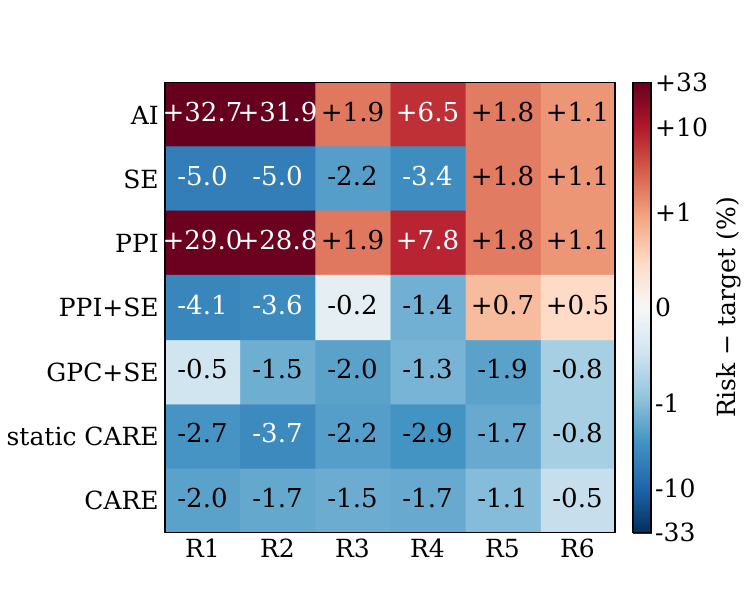}\par\nointerlineskip
		\caption{Final regional differences.}
		\label{fig:nuplan-regional-heatmaps}
	\end{subfigure}
	\caption{Regional model improvement and ensemble risk at the 5\% target, with 20\% preparation data for static methods.
		Radars show successive snapshots; heatmaps summarize the final snapshot on the same held-out panel of 2,000 scenarios.
		The risk target is highlighted as a dashed ring in the lower row, and the red band and outlines indicate exceedance.
		Region definition: open-loop time to collision below 1\,s for R1--R2, finite and at least 1\,s for R3--R4, and infinite for R5--R6; average displacement error below 5\,m for R1, R3, R5, and at least 5\,m for R2, R4, R6.
	}
	\label{fig:nuplan-regions}
\end{figure}

\textbf{Robustness and generality.}
\method performs consistently across different domains, AI models, and confidence measures.
On HateBench (\cref{tab:hatebench-models}), \method reduces human queries and rectified-model error across all eight black-box AI models, whose raw error rates range from 14.5\% to 44.7\%.
Query savings are 49--66\% w.r.t. SE and 26--48\% w.r.t. PPI+SE, with relative error reductions of 15.5--70.5\%.
On iSafetyBench (\cref{fig:isafety-confidence}), \method maintains similar query efficiency under direct (AI-provided) confidence, which is overconfident, and confidence inferred from the AI's predictive probability.
For SE and PPI variants, however, the query rate increases by 28.0 pp and 11.8+ pp when switching from inferred to direct confidence.
\method gains this robustness by also rectifying the confidence measure.

\section{Conclusion}

We introduced \method, an end-to-end framework for safe and efficient human--AI collaborative decision making. By combining model-aware residual learning with model-agnostic adaptive calibration, \method uses human feedback to improve future predictions while maintaining control of the final system's risk.
Experiments on real-world safety-critical tasks demonstrate substantial query savings across domains, AI models, and confidence measures.
Together, these results provide a theoretical and empirical foundation for safeguarding continuously improving AI models towards greater autonomy.

\bibliography{main}
\bibliographystyle{iclr2027_conference}

\newpage
\appendix

\section{Extended experimental setups and results} \label{apx:exp}

\textbf{Decisions, queries, and metrics.}
Each dataset comes with a binary decision task, with a fixed inexpensive predictor as ``AI'', and a benchmark reference label as the correct ``human'' response.
Ensemble risk divides final errors by all processed examples, including queried examples which contribute no error.
Query rate divides all oracle calls by the same denominator and includes initial offline static preparation dataset (training and/or calibration) and online adaptive escalations.
Reusing a purchased label for fitting or calibration does not incur a second query.
Rectified-model risk instead evaluates the predictor before oracle intervention, including examples that are ultimately queried.

\textbf{Paired evaluation.}
Each study uses ten randomized trials, pairing the data stream within trial across methods, risk targets, and preparation budgets. Means and standard errors across the ten trials are reported.
nuPlan and HateBench have reserved evaluation panels for assessing the rectified model.
This evaluation/test dataset does not enter the data stream, is not used for training or calibration, and thus is not counted in the query complexity.
Feature design may use covariate information, but no label information is revealed in this preparation stage.

\textbf{Methods.}
\begin{itemize}
  \item AI uses the original prediction without queries.
	\item PPI fits a linear rectifier on the static preparation data and then predicts without safety escalation.
	\item SE calibrates the original AI confidence.
	\item PPI+SE splits the static preparation data approximately equally between fitting the linear rectifier and calibrating escalation.
	\item GPC+SE replaces the linear rectifier in PPI+SE with a GP classifier, but otherwise uses the same static preparation and calibration.
	\item Static \method uses GP rectification and cross-fitted calibration on the shared static preparation data, then freezes the fitted policy.
	\item (Adaptive) \method does not have static preparation data. It starts with a conservative confidence threshold and learns from labels obtained through escalation.
	\item \method with PPI replaces the GP rectifier in \method with a linear one, but otherwise uses the same adaptive calibration and escalation.
\end{itemize}

\textbf{Prepare--offline--online schedule.}
Static methods reserves a preparation fraction of the entire data stream. This fraction is thus a hyperparameter for static methods, but not for (adaptive) \method.
Offline checkpoints are a hyperparameter for adaptive methods, and follow batches of 32, 64, 128, 256, 512, 1,024, and 2,048 newly queried labels, with the final interval repeated.

\textbf{Gap between theory and practice.}
Active escalation coincides with safety escalation in our experiments.
Computational approximations include sampled counterfactual replays using truncated horizons and logged acquisition decisions, and shorter EP iterations for the particle system.
Reported empirical risk only approximates the marginal risk guaranteed by \cref{thm:cali,cor:cali}.

\Cref{tab:apx-exp-protocol} summarizes the principal settings.

\begin{table}[th]
    \centering\small
    \setlength{\tabcolsep}{4pt}
    \caption{Shared protocol and dataset-specific settings.
			The default operating point uses 20\% preparation and a 5\% risk target.
		}
    \begin{tabular*}{\linewidth}{@{\extracolsep{\fill}}lrrrr@{}}
        \toprule
        & nuPlan & HateBench & iSafetyBench & ASIMOV \\
        \midrule
        Stream dataset size & 10,000 & 7,054 & 1,100 & 1,437 \\
        Reserved test panel & 2,000 & 784 & --- & --- \\
        Rolling-window size & 1,000 & 500 & 100 & 100 \\
        GP feature dimension & 256 & 64 & 256 & 256 \\
        Particle system size & 24 & 5 & 20 & 20 \\
				Preparation fractions (\%) & [10, 20, 30, 40] & [10, 20, 30, 40] & [10, 20, 30, 40] & [10, 20, 30, 40] \\
				Risk targets (\%) & [1, 3, 5, 7.5, 10] & [1, 3, 5, 7.5, 10] & [2.5, 5, 7.5, 10] & [5, 10, 15, 20] \\
        \bottomrule
    \end{tabular*}
    \label{tab:apx-exp-protocol}
\end{table}

\subsection{nuPlan} \label{apx:exp-nuplan}

\textbf{Task and evaluation.}
We use nuPlan driving scenarios \citep{caesar2022NuPlanclosedloop} with its inexpensive ``AI'' predictor as the open-loop simulator and its ``human'' oracle as the closed-loop oracle.
The binary task is to collision prediction, with the label being 1 if the closed-loop simulation has a time-to-collision (TTC) of less than one second.
The ``AI'' predictor's probability of safety is a probit transform of the open-loop TTC margin, with scale 0.5 seconds and probabilities clipped to $[10^{-6},1-10^{-6}]$.

The population contains 49,763 scenarios from six intersection-related scenario types and 63 driving logs.
A fixed, label-blind split reserves 9,953 scenarios, stratified by the six physical regions in \cref{tab:apx-nuplan-regions}; 39,810 remain eligible for streaming.
Each of ten paired trials randomly orders this pool and processes its first 10,000 scenarios.
Regional evaluation uses the same region-stratified panel of 2,000 reserved scenarios for every trial and method, with population region weights used for pooled snapshot metrics.
Driving logs occur in both partitions, so this evaluates within-population adaptation.

\textbf{Feature space.}
The rectifier receives ten engineered open-loop scalars and eight joint principal components of 384 ego-, agent-, and road-embedding coordinates.
The scalar and embedding blocks contribute 80\% and 20\% of normalized feature energy, respectively.
Representations are fitted without labels on each trial's first 2,000 covariates, frozen, and shared across methods.

\textbf{Risk--cost tradeoffs.}
\Cref{fig:apx-nuplan-sweeps,tab:apx-nuplan-grid} show the tradeoff across all five risk targets and four preparation budgets.
At the primary operating point, \method achieves 4.36\% ensemble risk with 26.6\% queries, compared with 50.3\% for SE, 45.3\% for PPI+SE, and 45.8\% for static \method.
These correspond to query reductions of 47\%, 41\%, and 42\%, respectively.
Mean \method risk is below every tested target, while query cost falls from 74.6\% at the 1\% target to 12.4\% at the 10\% target.
Larger static preparation budgets generally reduce ensemble risk but increase total query cost: buying more labels does not ensure that subsequent query savings repay their cost.
Adaptive \method avoids choosing a fixed preparation size and has lower mean query cost than these three comparators throughout the reported grid.

\begin{figure}[th]
    \centering
    \hfil\begin{subfigure}[t]{0.32\linewidth}
        \includegraphics[width=\linewidth]{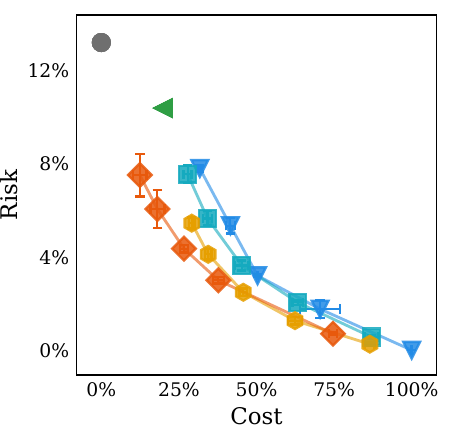}
        \caption{Risk-target sweep.}
    \end{subfigure}\hfil
    \begin{subfigure}[t]{0.32\linewidth}
        \includegraphics[width=\linewidth]{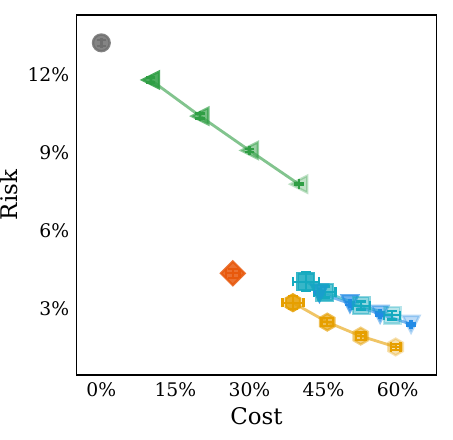}
        \caption{Preparation-budget sweep.}
    \end{subfigure}\hfil
    \par\vspace{2pt}
    \includegraphics[width=\linewidth]{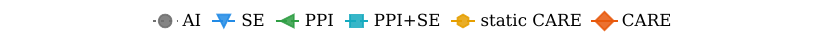}
    \caption{nuPlan risk--cost tradeoffs. (a) Targets of 1\%, 3\%, 5\%, 7.5\%, and 10\%, with 20\% static preparation. (b) 10\%, 20\%, 30\%, and 40\% preparation at the 5\% target, shown with decreasing marker opacity.}
    \label{fig:apx-nuplan-sweeps}
\end{figure}

\begin{table}[t]
\centering
\small
\setlength{\tabcolsep}{5pt}
\caption{nuPlan risk--cost tradeoffs. Ensemble risk / oracle query rate (\%) over the complete target and static-preparation grid.}
\begin{tabular*}{\linewidth}{@{\extracolsep{\fill}}rrcccc@{}}
\toprule
Target & Prep. & SE & PPI+SE & static \method & \method \\
\midrule
1\% & 10\% & 0.00 / 100.0 & 0.57 / 87.0 & 0.43 / 84.3 & 0.70 / 74.6 \\
1\% & 20\% & 0.00 / 100.0 & 0.56 / 87.1 & 0.26 / 86.5 & 0.70 / 74.6 \\
1\% & 30\% & 0.00 / 100.0 & 0.55 / 87.2 & 0.16 / 89.3 & 0.70 / 74.6 \\
1\% & 40\% & 0.00 / 100.0 & 0.47 / 89.6 & 0.10 / 91.5 & 0.70 / 74.6 \\
\addlinespace[2pt]
3\% & 10\% & 2.04 / 66.0 & 2.46 / 56.9 & 1.85 / 55.4 & 3.00 / 37.7 \\
3\% & 20\% & 1.76 / 70.5 & 2.05 / 63.3 & 1.27 / 62.4 & 3.00 / 37.7 \\
3\% & 30\% & 1.28 / 78.2 & 1.76 / 68.4 & 0.97 / 67.2 & 3.00 / 37.7 \\
3\% & 40\% & 1.15 / 81.0 & 1.55 / 72.1 & 0.72 / 72.8 & 3.00 / 37.7 \\
\addlinespace[2pt]
5\% & 10\% & 3.60 / 44.2 & 4.04 / 41.4 & 3.23 / 38.8 & 4.36 / 26.6 \\
5\% & 20\% & 3.20 / 50.3 & 3.63 / 45.3 & 2.48 / 45.8 & 4.36 / 26.6 \\
5\% & 30\% & 2.79 / 56.5 & 3.12 / 52.6 & 1.95 / 52.5 & 4.36 / 26.6 \\
5\% & 40\% & 2.40 / 62.7 & 2.75 / 58.9 & 1.53 / 59.6 & 4.36 / 26.6 \\
\addlinespace[2pt]
7.5\% & 10\% & 5.73 / 35.4 & 6.05 / 28.9 & 4.98 / 27.5 & 6.06 / 18.1 \\
7.5\% & 20\% & 5.34 / 41.6 & 5.65 / 34.2 & 4.11 / 34.4 & 6.06 / 18.1 \\
7.5\% & 30\% & 4.38 / 49.9 & 4.88 / 42.1 & 3.32 / 42.8 & 6.06 / 18.1 \\
7.5\% & 40\% & 4.15 / 55.4 & 4.28 / 49.8 & 2.73 / 51.1 & 6.06 / 18.1 \\
\addlinespace[2pt]
10\% & 10\% & 8.88 / 22.9 & 8.04 / 21.5 & 6.73 / 20.7 & 7.51 / 12.4 \\
10\% & 20\% & 7.80 / 31.7 & 7.55 / 27.8 & 5.45 / 29.1 & 7.51 / 12.4 \\
10\% & 30\% & 6.77 / 40.3 & 6.49 / 37.0 & 4.65 / 37.7 & 7.51 / 12.4 \\
10\% & 40\% & 5.84 / 48.8 & 5.81 / 45.3 & 3.82 / 47.1 & 7.51 / 12.4 \\
\bottomrule
\end{tabular*}
\label{tab:apx-nuplan-grid}
\end{table}

\textbf{Adaptive learning.}
\Cref{fig:apx-nuplan-learning} presents adaptive learning through model improvement, escalation cost, and final-ensemble error.
In the first 1,000 arrivals, \method queries 43.8\% and has 13.1\% model risk; in the final 1,000 arrivals these become 25.6\% and 8.6\%.
Static \method ends at 32.0\% queries and 11.1\% model risk.
Continued residual learning therefore improves decisions before intervention and reduces the amount of intervention needed.
The final-window ensemble risk is 3.7\% for \method.
Short-window risk may fluctuate around the nominal target; these trajectories describe successive deployed policies rather than a guarantee for every realized window.

\begin{figure}[th]
    \centering
    \begin{subfigure}[t]{0.32\linewidth}
        \includegraphics[width=\linewidth]{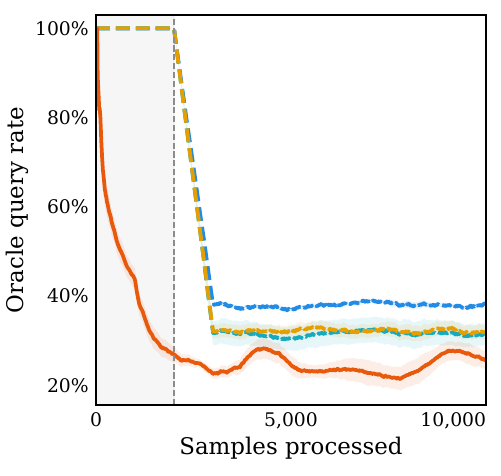}
    \end{subfigure}\hfill
    \begin{subfigure}[t]{0.32\linewidth}
        \includegraphics[width=\linewidth]{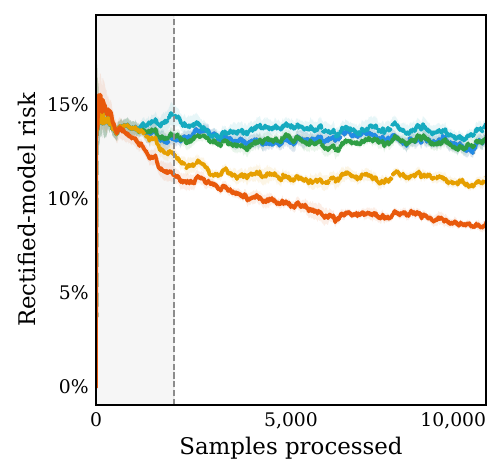}
    \end{subfigure}\hfill
    \begin{subfigure}[t]{0.32\linewidth}
        \includegraphics[width=\linewidth]{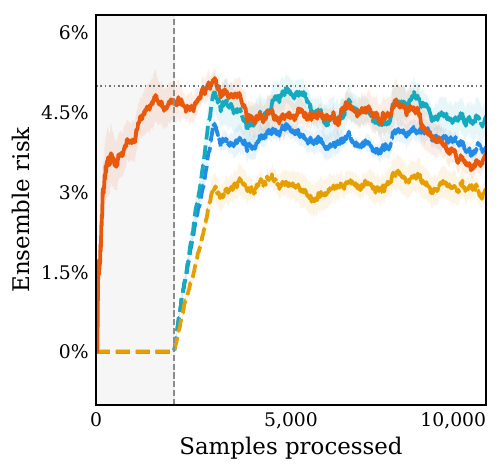}
    \end{subfigure}
    \includegraphics[width=\linewidth]{fig/apx-method-legend.pdf}
    \caption{nuPlan learning trajectories at the 5\% target and 20\% static preparation. Lines average trailing windows of up to 1,000 arrivals over ten trials. Gray shading marks the 2,000-arrival static preparation period. AI and PPI are shown as uncalibrated references in the model-risk panel. The dotted horizontal line is the nominal ensemble-risk target.}
    \label{fig:apx-nuplan-learning}
\end{figure}

\textbf{Local adaptation.}
The frozen-model analysis in \cref{fig:nuplan-regions} isolates spatial variation from changes in the arriving stream.
\Cref{tab:apx-nuplan-regions} supplies the region definitions, evaluation support, and final \method values.
The strongest correction occurs in R1 and R2, where open-loop TTC is below one second: model risk falls from 37.7\% to 13.6\% and from 36.9\% to 15.1\%, respectively.
These regions also receive the most oracle queries, reflecting their greater remaining difficulty.
On the full weighted panel, \method reduces model risk from 12.7\% initially to 10.2\% after 2,000 decisions and 7.9\% after 10,000; static \method remains at 10.4\% after preparation.
Final mean regional ensemble risks for \method range from 3.0\% to 4.5\%, all below the 5\% target.
As in the main text, GPC+SE also remains below the target in every region, whereas SE and PPI+SE exceed it in R5 and R6.
This is an empirical regional balance result; the calibration target concerns marginal risk, and the smaller regions have correspondingly larger sampling uncertainty.

\begin{table}[t]
\centering
\small
\setlength{\tabcolsep}{3pt}
\begin{tabular*}{\linewidth}{@{\extracolsep{\fill}}clrrrrr@{}}
\toprule
& Region & Panel & AI risk & \multicolumn{3}{c}{Final \method (\%)} \\
\cmidrule(lr){5-7}
& & size & (\%) & Model risk & Ensemble risk & Queries \\
\midrule
R1 & TTC $<1$, ADE $<5$ & 53 & 37.7 & $13.6\pm1.5$ & $3.0\pm1.1$ & $49.6\pm5.0$ \\
R2 & TTC $<1$, ADE $\ge5$ & 293 & 36.9 & $15.1\pm0.5$ & $3.3\pm0.4$ & $56.5\pm3.7$ \\
R3 & Finite TTC $\ge1$, ADE $<5$ & 320 & 6.9 & $5.1\pm0.2$ & $3.5\pm0.2$ & $13.8\pm1.9$ \\
R4 & Finite TTC $\ge1$, ADE $\ge5$ & 374 & 11.5 & $8.1\pm0.5$ & $3.3\pm0.3$ & $30.6\pm2.8$ \\
R5 & Infinite TTC, ADE $<5$ & 468 & 6.8 & $6.3\pm0.2$ & $3.9\pm0.3$ & $11.9\pm1.4$ \\
R6 & Infinite TTC, ADE $\ge5$ & 492 & 6.1 & $6.1\pm0.1$ & $4.5\pm0.3$ & $10.1\pm1.6$ \\
\bottomrule
\end{tabular*}
\caption{nuPlan regional definitions and final frozen-model results at the 5\% target. TTC (time to collision) is in seconds and ADE (average displacement error) is in meters.}
\label{tab:apx-nuplan-regions}
\end{table}

\subsection{HateBench}
\label{apx:exp-hatebench}

\textbf{Task and evaluation.}
We use HateBench responses \citep{shen2025HateBenchBenchmarking} with eight cached detectors as the ``AI'' predictors and released human hate/non-hate annotations as the oracle labels.
The population contains 7,838 responses; a fixed, label-blind split reserves 784 responses by normalized-text group, leaving 7,054 for each trial.
The reserved panel and stream orders are shared across detectors.
Prompt families may occur in both partitions, so this evaluates within-population adaptation.
Perspective and both Detoxify variants use the identity-attack score at threshold 0.5; the other detectors retain their native binary decisions with scores oriented toward the hate class.
The confidence supplied to escalation is the absolute margin from 0.5, scaled to $[0,1]$.

\textbf{Feature space.}
The rectifier receives response-text TF--IDF features with unigrams and bigrams, reduced to 32 singular-vector coordinates, together with log response length and the selected detector's available category scores.
Coordinates are standardized, clipped at six, and divided by the square root of the feature dimension.
Representations are fitted without labels on each trial's first 20\% of covariates, frozen, and shared across methods.
The same fixed GP prior is used across all eight detectors.

\textbf{Risk--cost tradeoffs.}
\Cref{fig:apx-hatebench-sweeps} show the tradeoff across all five risk targets for all models, and selectively, \cref{tab:apx-hatebench-grid-perspective} reports the tradeoff across four preparation budgets for Perspective.
With Perspective, at the 5\% target and 10\% static preparation, \method achieves 3.71\% ensemble risk with 21.5\% queries, compared with 54.0\% for SE, 31.8\% for PPI+SE, and 28.3\% for static \method.
Query cost falls from 42.2\% at the 1\% target to 13.7\% at the 10\% target.
Larger static preparation budgets generally reduce ensemble risk but increase total query cost.
Adaptive \method has lower mean query cost than these three comparators throughout the reported grid.
Similar trends hold with other models.

\begin{figure}[th]
\centering
\begin{subfigure}[t]{0.24\linewidth}\centering
\includegraphics[width=\linewidth]{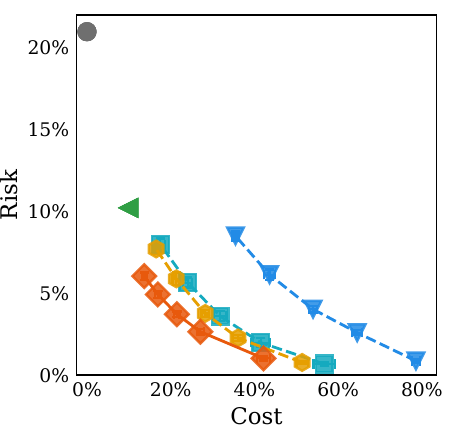}
\caption{Perspective.}\end{subfigure}
\hfill
\begin{subfigure}[t]{0.24\linewidth}\centering
\includegraphics[width=\linewidth]{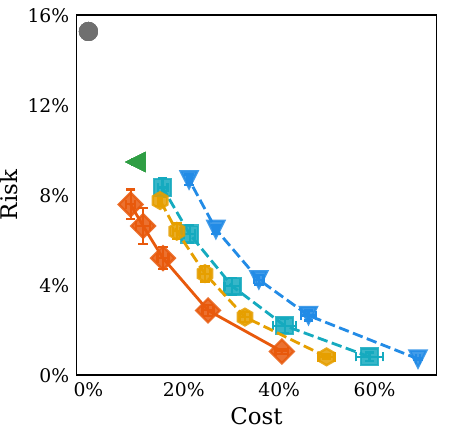}
\caption{Moderation.}\end{subfigure}
\begin{subfigure}[t]{0.24\linewidth}\centering
\includegraphics[width=\linewidth]{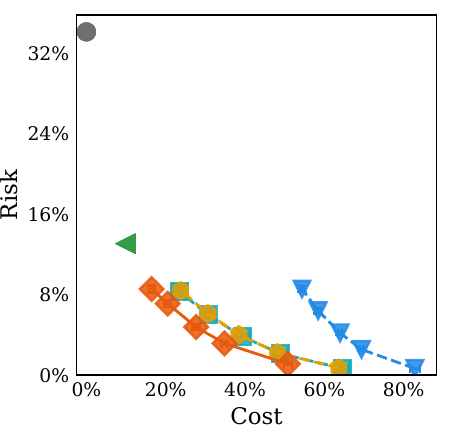}
\caption{Detoxify Original.}\end{subfigure}
\hfill
\begin{subfigure}[t]{0.24\linewidth}\centering
\includegraphics[width=\linewidth]{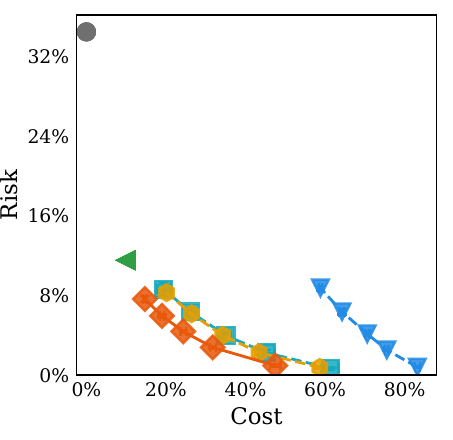}
\caption{Detoxify Unbiased.}\end{subfigure}

\begin{subfigure}[t]{0.24\linewidth}\centering
\includegraphics[width=\linewidth]{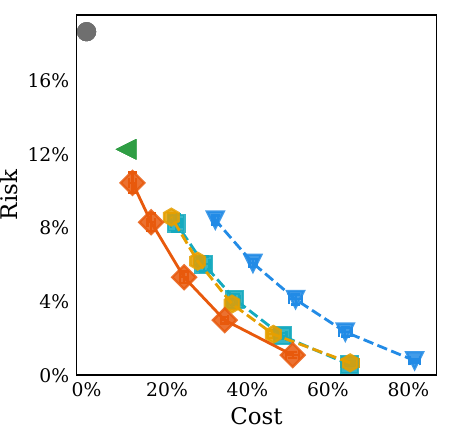}
\caption{LFTW.}\end{subfigure}
\hfill
\begin{subfigure}[t]{0.24\linewidth}\centering
\includegraphics[width=\linewidth]{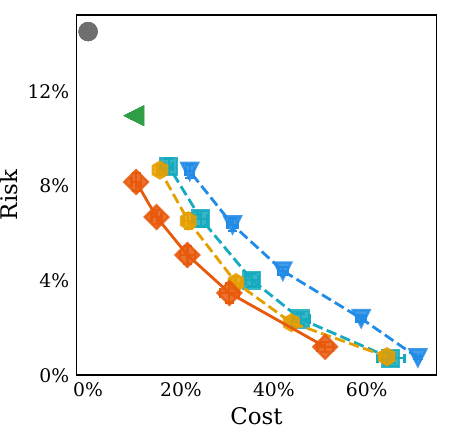}
\caption{TweetHate.}\end{subfigure}
\begin{subfigure}[t]{0.24\linewidth}\centering
\includegraphics[width=\linewidth]{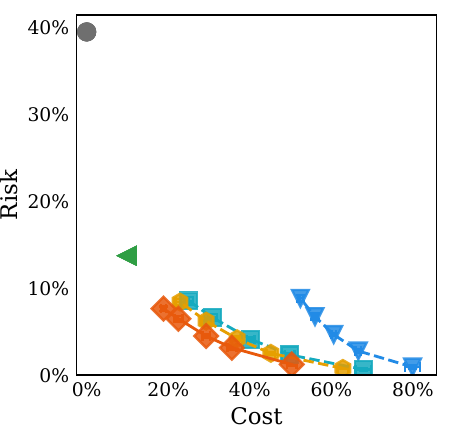}
\caption{HSBERT.}\end{subfigure}
\hfill
\begin{subfigure}[t]{0.24\linewidth}\centering
\includegraphics[width=\linewidth]{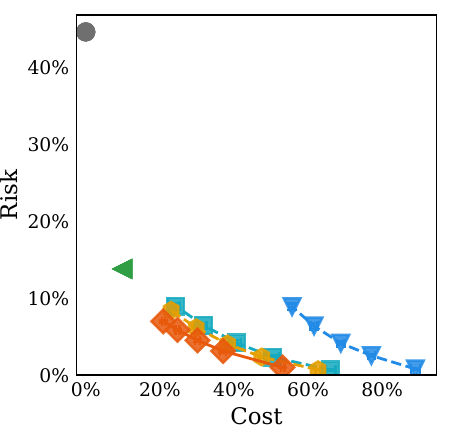}
\caption{BERTHateXplain.}\end{subfigure}
\includegraphics[width=\linewidth]{fig/apx-method-legend.pdf}

\caption{HateBench risk--cost tradeoffs with different models. Targets of 1\%, 3\%, 5\%, 7.5\%, and 10\%, with 10\% static preparation.}
\label{fig:apx-hatebench-sweeps}
\end{figure}

\begin{table}[t]\centering\small
\caption{HateBench risk--cost tradeoffs with Perspective. Ensemble risk / oracle query rate (\%) over the complete target and static-preparation grid.}
\label{tab:apx-hatebench-grid-perspective}
\setlength{\tabcolsep}{4pt}\renewcommand{\arraystretch}{1.05}
\begin{tabular*}{\linewidth}{@{\extracolsep{\fill}}rrcccc@{}}\toprule
Target & Prep. & SE & PPI+SE & static \method & \method \\ \midrule
1\% & 10\% & 0.86 / 78.6 & 0.67 / 56.7 & 0.77 / 51.4 & 1.02 / 42.2 \\
1\% & 20\% & 0.81 / 80.3 & 0.54 / 57.5 & 0.69 / 53.8 & 1.02 / 42.2 \\
1\% & 30\% & 0.70 / 82.8 & 0.57 / 59.8 & 0.60 / 58.6 & 1.02 / 42.2 \\
1\% & 40\% & 0.59 / 85.3 & 0.53 / 64.1 & 0.53 / 64.0 & 1.02 / 42.2 \\
\addlinespace[2pt]
3\% & 10\% & 2.58 / 64.6 & 1.97 / 41.5 & 2.26 / 36.1 & 2.64 / 27.1 \\
3\% & 20\% & 2.31 / 68.2 & 2.01 / 41.8 & 2.02 / 40.7 & 2.64 / 27.1 \\
3\% & 30\% & 1.98 / 72.5 & 1.66 / 48.5 & 1.73 / 47.3 & 2.64 / 27.1 \\
3\% & 40\% & 1.73 / 76.0 & 1.58 / 54.8 & 1.62 / 54.0 & 2.64 / 27.1 \\
\addlinespace[2pt]
5\% & 10\% & 4.01 / 54.0 & 3.59 / 31.8 & 3.75 / 28.3 & 3.71 / 21.5 \\
5\% & 20\% & 3.70 / 58.3 & 3.56 / 34.2 & 3.50 / 33.8 & 3.71 / 21.5 \\
5\% & 30\% & 3.20 / 63.6 & 3.20 / 41.2 & 3.04 / 41.4 & 3.71 / 21.5 \\
5\% & 40\% & 2.87 / 68.1 & 2.64 / 49.7 & 2.78 / 48.9 & 3.71 / 21.5 \\
\addlinespace[2pt]
7.5\% & 10\% & 6.11 / 43.6 & 5.63 / 24.1 & 5.86 / 21.3 & 4.91 / 16.9 \\
7.5\% & 20\% & 5.50 / 49.2 & 5.45 / 28.1 & 5.36 / 27.8 & 4.91 / 16.9 \\
7.5\% & 30\% & 4.79 / 55.6 & 4.94 / 35.7 & 4.67 / 36.1 & 4.91 / 16.9 \\
7.5\% & 40\% & 4.22 / 61.4 & 4.24 / 44.8 & 4.29 / 44.4 & 4.91 / 16.9 \\
\addlinespace[2pt]
10\% & 10\% & 8.48 / 35.5 & 7.95 / 17.6 & 7.68 / 16.5 & 6.03 / 13.7 \\
10\% & 20\% & 7.68 / 42.3 & 7.36 / 23.2 & 6.84 / 24.0 & 6.03 / 13.7 \\
10\% & 30\% & 6.75 / 49.4 & 6.77 / 31.4 & 5.94 / 32.9 & 6.03 / 13.7 \\
10\% & 40\% & 5.82 / 56.5 & 5.73 / 41.3 & 5.19 / 42.2 & 6.03 / 13.7 \\
\bottomrule\end{tabular*}\end{table}

\textbf{Adaptive learning.}
\Cref{fig:apx-hatebench-learning} presents adaptive learning with Perspective through model improvement, escalation cost, and final-ensemble error.
In the first 500 arrivals, \method queries 41.4\% and has 14.3\% model risk; in the final 500 arrivals these become 16.8\% and 9.4\%.
Continued residual learning therefore improves decisions before intervention and reduces the amount of intervention needed.
The final-window ensemble risk is 3.7\% for \method.

\begin{figure}[th]
    \centering
    \begin{subfigure}[t]{0.32\linewidth}
        \includegraphics[width=\linewidth]{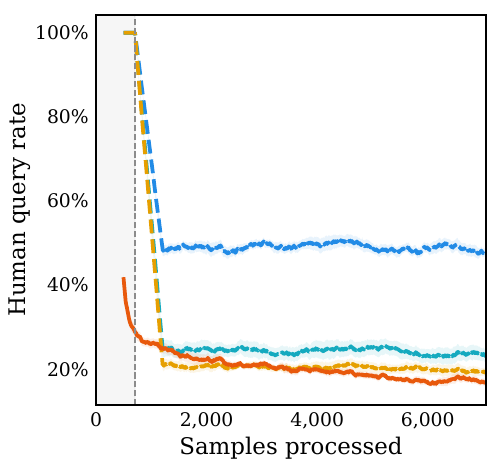}
    \end{subfigure}\hfill
    \begin{subfigure}[t]{0.32\linewidth}
        \includegraphics[width=\linewidth]{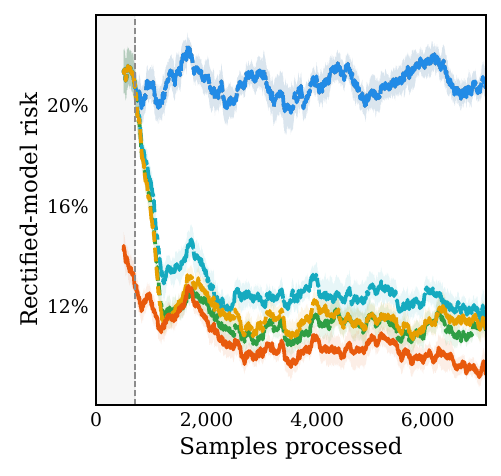}
    \end{subfigure}\hfill
    \begin{subfigure}[t]{0.32\linewidth}
        \includegraphics[width=\linewidth]{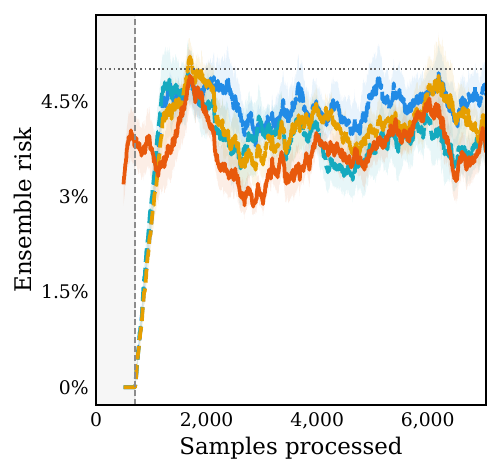}
    \end{subfigure}
    \includegraphics[width=\linewidth]{fig/apx-method-legend.pdf}
    \caption{HateBench learning trajectories with Perspective at the 5\% target and 10\% static preparation. Lines average complete trailing windows of 500 arrivals over ten trials. Gray shading marks the 705-arrival static preparation period. AI and PPI are shown as uncalibrated references in the model-risk panel. The dotted horizontal line is the nominal ensemble-risk target.}
    \label{fig:apx-hatebench-learning}
\end{figure}

\textbf{Robustness across AI models.}
\Cref{tab:hatebench-models} summarize query cost and held-out model improvement across all eight detectors.
At the 5\% target and 10\% static preparation, \method uses 15.6--30.2\% queries with mean ensemble risks of 3.71--5.31\%.
These correspond to query reductions of 49--66\% relative to SE and 26--48\% relative to PPI+SE.
On the reserved panel, raw model risk ranges from 14.2\% to 47.3\%, while final \method model risk ranges from 10.8\% to 14.6\%.
The relative reductions of 15.5--70.5\% show that residual learning improves substantially different initial predictors.

\begin{table}[t]
	\centering
	\caption{HateBench comparison across AI models at the 5\% target and 10\% static preparation. All values are percentages. Raw AI error is measured on the stream; model-error reduction is measured on the reserved panel.}
	\label{tab:hatebench-models}
	\small
	\setlength{\tabcolsep}{2pt}
	\renewcommand{\arraystretch}{1.1}
	\begin{tabular*}{\linewidth}{@{\extracolsep{\fill}}lrrrrrrr@{}}
		\toprule
		\multirow{2}{*}{AI model} & \multirow{2}{*}{\makecell{Raw AI\\error}} &
		\multicolumn{3}{c}{Human query rate} & \multicolumn{3}{c}{Relative improvement} \\
		\cmidrule(lr){3-5}\cmidrule(l){6-8}
		& & SE & PPI+SE & \method & \makecell{Query saving\\vs SE $\uparrow$} &
		\makecell{Query saving\\vs PPI+SE $\uparrow$} &
		\makecell{Model-error\\reduction $\uparrow$} \\
		\midrule
		Perspective & 20.9 & $54.0 \pm 0.9$ & $31.8 \pm 1.5$ & $21.5 \pm 0.7$ & 60.2 & 32.5 & 40.5 \\
		Moderation & 15.3 & $35.7 \pm 1.0$ & $30.2 \pm 1.7$ & $15.6 \pm 1.2$ & 56.3 & 48.2 & 15.5 \\
		Detoxify Original & 34.1 & $64.0 \pm 0.5$ & $39.3 \pm 1.5$ & $27.7 \pm 1.0$ & 56.7 & 29.6 & 58.6 \\
		Detoxify Unbiased & 34.5 & $70.6 \pm 0.6$ & $35.1 \pm 1.3$ & $24.3 \pm 0.7$ & 65.5 & 30.6 & 65.8 \\
		LFTW & 18.6 & $52.0 \pm 1.6$ & $36.9 \pm 1.2$ & $24.2 \pm 1.0$ & 53.4 & 34.3 & 27.2 \\
		TweetHate & 14.5 & $41.9 \pm 1.1$ & $35.3 \pm 1.4$ & $21.4 \pm 1.6$ & 49.0 & 39.3 & 23.4 \\
		HSBERT & 39.4 & $60.6 \pm 0.5$ & $40.3 \pm 0.9$ & $29.3 \pm 0.9$ & 51.6 & 27.2 & 65.7 \\
		BERTHateXplain & 44.7 & $68.8 \pm 1.0$ & $40.7 \pm 1.2$ & $30.2 \pm 0.7$ & 56.2 & 25.9 & 70.5 \\
		\bottomrule
	\end{tabular*}
\end{table}

\subsection{iSafetyBench}
\label{apx:exp-isafety}

\textbf{Task and evaluation.}
We use iSafetyBench videos \citep{abdullah2025iSafetyBenchvideolanguage} for binary workplace hazard detection, with the frozen Qwen3-VL-2B-Instruct model as the ``AI'' predictor and released category annotations as the oracle labels.
The population contains 1,100 clips: 420 hazardous and 680 routine.
The AI receives 4--32 frames sampled across each clip, without audio.
Its hazard probability $p_A$ is obtained from the two answer logits, and a separate Yes/No follow-up supplies direct confidence $c_{\mathrm{direct}}=P(\text{Yes})$.
We compare this with inferred confidence $c_{\mathrm{inferred}}=2|p_A-1/2|$, holding the video inputs and initial AI decisions fixed.
Each trial processes all 1,100 clips from 151 source-video families, grouped within calibration folds.

\textbf{Feature space.}
The rectifier receives the 2,048-dimensional final-layer prompt state, reduced to 32 principal components.
Representations are fitted without labels on each trial's first 20\% of covariates, frozen, and shared across methods and confidence measures.

\textbf{Risk--cost tradeoffs.}
\Cref{fig:apx-isafety-sweeps} show the tradeoff across all four risk targets and four preparation budgets under both confidence measures.
At the 5\% target and 10\% static preparation, direct-confidence \method achieves 3.79\% ensemble risk with 21.7\% queries, compared with 57.1\% for SE and 26.7\% for static \method.
With inferred confidence, it achieves 2.86\% risk with 24.8\% queries, compared with 35.1\% and 27.4\% for these baselines.
Mean \method risk is below every tested target under both confidence measures.
Larger static preparation budgets generally reduce ensemble risk but increase total query cost: static \method's query rate rises from 26.7\% to 50.2\% across the direct-confidence budget sweep at the 5\% target.

\begin{figure}[th]
\centering
\begin{subfigure}[t]{0.24\linewidth}
\includegraphics[width=\linewidth]{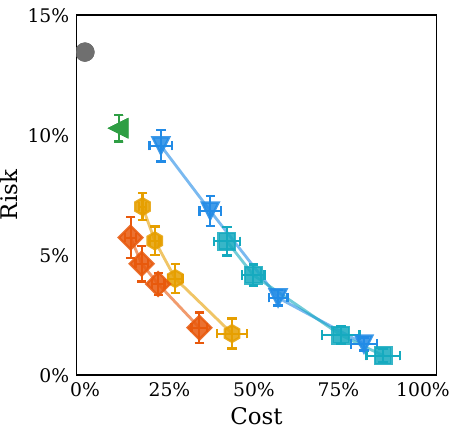}
\caption{Direct: risk sweep.}
\end{subfigure}\hfill
\begin{subfigure}[t]{0.24\linewidth}
\includegraphics[width=\linewidth]{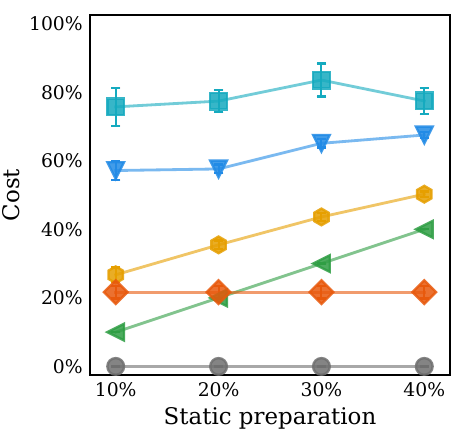}
\caption{Direct: prep. sweep.}
\end{subfigure}
\begin{subfigure}[t]{0.24\linewidth}
\includegraphics[width=\linewidth]{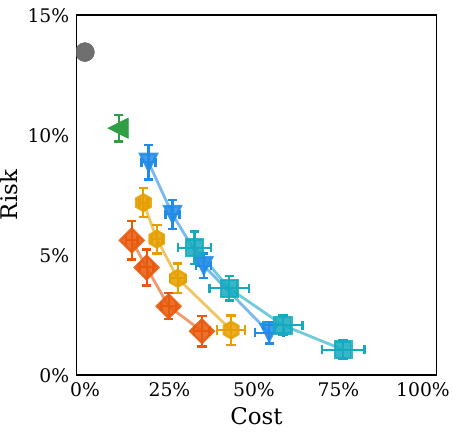}
\caption{Inferred: risk sweep.}
\end{subfigure}\hfill
\begin{subfigure}[t]{0.24\linewidth}
\includegraphics[width=\linewidth]{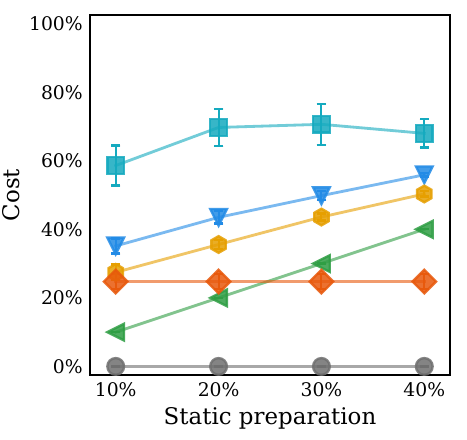}
\caption{Inferred: prep. sweep.}
\end{subfigure}
\includegraphics[width=\linewidth]{fig/apx-method-legend.pdf}
\caption{iSafetyBench risk--cost tradeoffs. (a, c) Targets of 2.5\%, 5\%, 7.5\%, and 10\%, with 10\% static preparation. (b, d) 10\%, 20\%, 30\%, and 40\% preparation at the 5\% target.}
\label{fig:apx-isafety-sweeps}
\end{figure}

\textbf{Adaptive learning.}
\Cref{fig:apx-isafety-learning} presents adaptive learning through model improvement, escalation cost, and final-ensemble error.
With direct confidence, \method queries 25.3\% and has 9.3\% model risk in the window ending at 300 clips; in the final 100 clips these become 6.9\% and 4.8\%.
Static \method ends at 16.0\% queries and 9.4\% model risk.
Continued residual learning therefore improves decisions before intervention and reduces the amount of intervention needed.
The inferred-confidence results show a similar reduction in model risk, reaching 4.5\% in the final window.

\begin{figure}[th]
\centering
\begin{subfigure}[t]{0.32\linewidth}
\includegraphics[width=\linewidth]{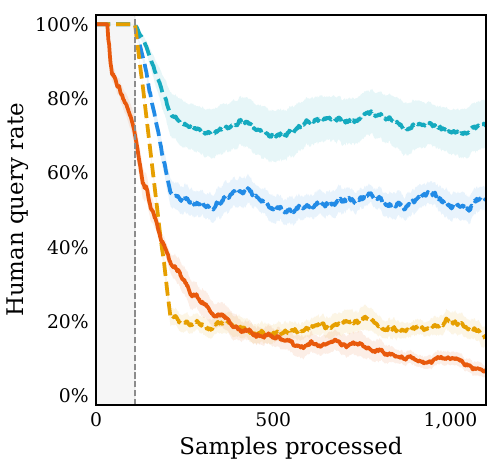}
\end{subfigure}\hfill
\begin{subfigure}[t]{0.32\linewidth}
\includegraphics[width=\linewidth]{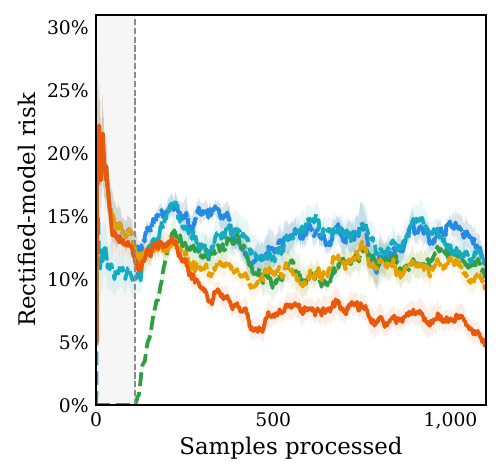}
\caption{Direct confidence.}
\end{subfigure}\hfill
\begin{subfigure}[t]{0.32\linewidth}
\includegraphics[width=\linewidth]{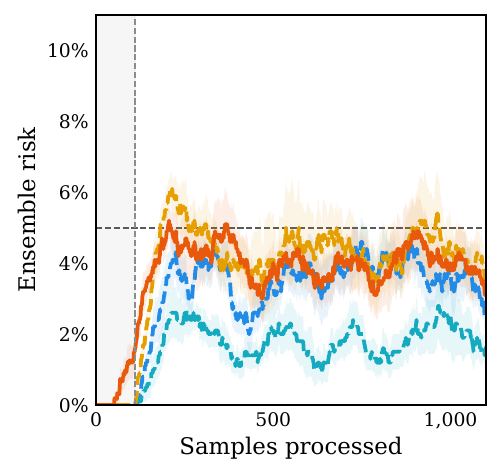}
\end{subfigure}
\par\medskip
\begin{subfigure}[t]{0.32\linewidth}
\includegraphics[width=\linewidth]{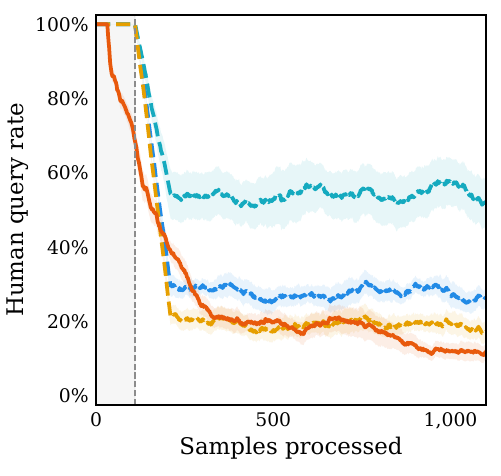}
\end{subfigure}\hfill
\begin{subfigure}[t]{0.32\linewidth}
\includegraphics[width=\linewidth]{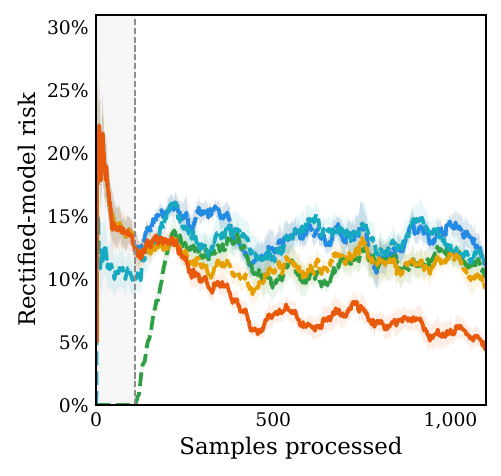}
\caption{Inferred confidence.}
\end{subfigure}\hfill
\begin{subfigure}[t]{0.32\linewidth}
\includegraphics[width=\linewidth]{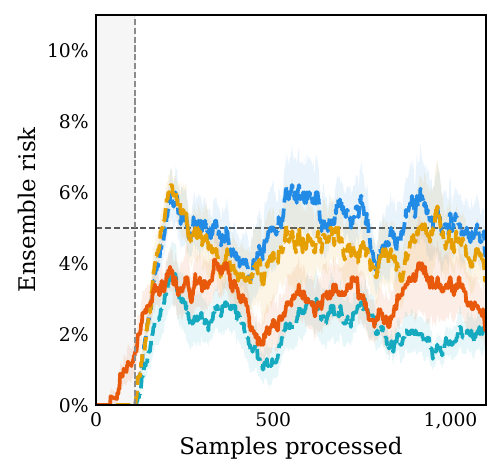}
\end{subfigure}
\includegraphics[width=\linewidth]{fig/apx-method-legend.pdf}
\caption{iSafetyBench learning trajectories at the 5\% target and 10\% static preparation. Lines average trailing windows of up to 100 arrivals over ten trials. Gray shading marks the 110-arrival static preparation period, where static model risks can include in-sample predictions. AI and PPI are shown as uncalibrated references in the model-risk panels. The dashed horizontal line is the nominal ensemble-risk target.}
\label{fig:apx-isafety-learning}
\end{figure}

\textbf{Robustness to confidence choice.}
Inferred confidence ranks the original AI's correct decisions more accurately than direct confidence (correctness AUROC 0.803 versus 0.642).
\Cref{tab:apx-isafety-confidence} extends the paired comparison in \cref{fig:isafety-confidence} to all four risk targets.
At 2.5\%, replacing inferred with direct confidence increases query rates by 28.0 percentage points for SE, 11.8 for PPI+SE, and 14.4 for \method with PPI, while GP \method changes by only $-0.7$ points.
Across the four targets, GP \method's absolute mean query-rate change is at most 3.1 points.
Rectifying confidence together with predictions reduces reliance on the original confidence measure.

\begin{table}[t]
\centering\small
\setlength{\tabcolsep}{2.5pt}
\caption{Effect of confidence choice on iSafetyBench with 10\% static preparation. Entries are paired mean differences (direct minus inferred), in percentage points, for ensemble risk ($\Delta R$) and human query rate ($\Delta Q$).}
\label{tab:apx-isafety-confidence}
\begin{tabular*}{\linewidth}{@{\extracolsep{\fill}}lrrrrrrrr@{}}
\toprule
& \multicolumn{2}{c}{2.5\% target} & \multicolumn{2}{c}{5\% target} & \multicolumn{2}{c}{7.5\% target} & \multicolumn{2}{c}{10\% target}\\
\cmidrule(lr){2-3}\cmidrule(lr){4-5}\cmidrule(lr){6-7}\cmidrule(l){8-9}
Method & $\Delta R$ & $\Delta Q$ & $\Delta R$ & $\Delta Q$ & $\Delta R$ & $\Delta Q$ & $\Delta R$ & $\Delta Q$\\
\midrule
SE & $-0.5$ & $+28.0$ & $-1.3$ & $+22.0$ & $+0.1$ & $+11.2$ & $+0.7$ & $+3.7$\\
PPI+SE & $-0.2$ & $+11.8$ & $-0.4$ & $+17.1$ & $+0.6$ & $+7.1$ & $+0.3$ & $+9.6$\\
GPC+SE & $+0.1$ & $-1.2$ & $-0.2$ & $+1.1$ & $-0.3$ & $+0.4$ & $+0.1$ & $-0.7$\\
static \method & $-0.1$ & $+0.3$ & $+0.0$ & $-0.7$ & $-0.1$ & $-0.6$ & $-0.2$ & $-0.3$\\
\method with PPI & $+0.2$ & $+14.4$ & $-0.6$ & $+5.6$ & $-0.5$ & $+2.1$ & $+0.0$ & $-3.0$\\
\method & $+0.1$ & $-0.7$ & $+0.9$ & $-3.1$ & $+0.2$ & $-1.5$ & $+0.1$ & $-0.4$\\
\bottomrule
\end{tabular*}
\end{table}

\subsection{ASIMOV}
\label{apx:exp-asimov}

\textbf{Task and evaluation.}
We use ASIMOV-Agentic queries \citep{gemini2026roboticsSafety} for binary robot dispatch decisions, with the frozen Qwen3-VL-2B-Instruct model as the ``AI'' predictor and benchmark reference decisions as the oracle labels.
The population contains 1,437 queries: 641 on physical constraints, 396 on instruction ambiguity, and 400 on VLA feasibility.
The binary task is deciding whether to dispatch the instruction for immediate execution.
Physical-constraint references are human annotated; the other subsets use the released benchmark references.
The AI receives the image, instruction, and tool schemas, and its decision probability is obtained from the two answer logits.
Direct and inferred confidence follow the same construction as in iSafetyBench, holding inputs and initial AI decisions fixed.
Each trial processes all 1,437 queries from 722 distinct images, grouped within calibration folds.

\textbf{Feature space.}
The rectifier receives 13,312 image-and-text features from decoder and vision hidden states, reduced to 64 principal components.
These coordinates and the signed Yes-minus-No logit margin contribute 75\% and 25\% of normalized feature energy, respectively.
Representations are fitted without labels on each trial's first 20\% of covariates, frozen, and shared across methods and confidence measures.

\textbf{Risk--cost tradeoffs.}
\Cref{fig:apx-asimov-sweeps} show the tradeoff across all four risk targets and four preparation budgets under both confidence measures.
At the 5\% target and 20\% static preparation with inferred confidence, \method achieves 4.76\% ensemble risk with 50.8\% queries, compared with 90.5\% for SE, 95.3\% for PPI+SE, and 60.3\% for static \method.
Mean \method risk is below every tested target under both confidence measures.
Larger static preparation budgets do not always reduce total query cost: static \method uses 75.4\%, 60.3\%, 63.3\%, and 66.3\% queries across the four budgets at the 5\% target with inferred confidence.
Adaptive \method avoids choosing a fixed preparation size.

\begin{figure}[th]
    \centering
    \begin{subfigure}[t]{0.24\linewidth}
        \centering
        \includegraphics[width=\linewidth]{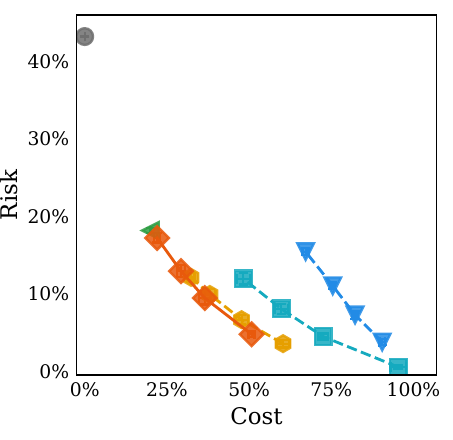}
        \caption{Inferred confidence.}
    \end{subfigure}\hfill
    \begin{subfigure}[t]{0.24\linewidth}
        \centering
        \includegraphics[width=\linewidth]{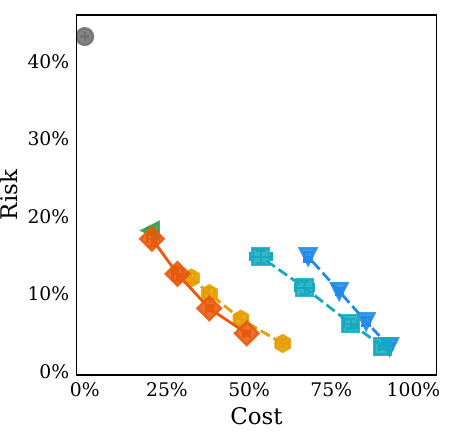}
        \caption{Direct confidence.}
    \end{subfigure}
    \begin{subfigure}[t]{0.24\linewidth}
        \centering
        \includegraphics[width=\linewidth]{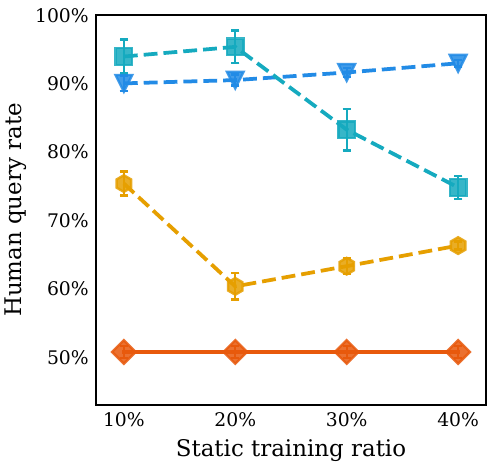}
        \caption{Queries vs preparation.}
    \end{subfigure}\hfill
    \begin{subfigure}[t]{0.24\linewidth}
        \centering
        \includegraphics[width=\linewidth]{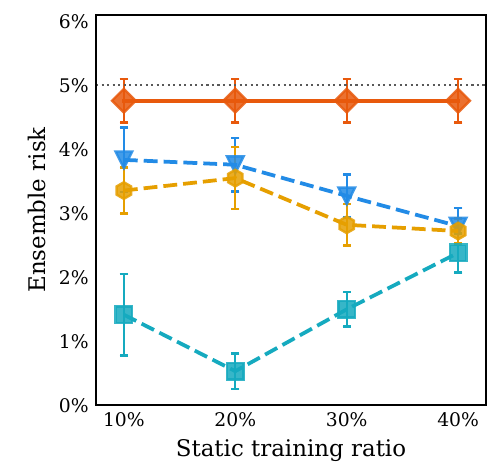}
        \caption{Risk vs preparation.}
    \end{subfigure}
    \includegraphics[width=\linewidth]{fig/apx-method-legend.pdf}
    \caption{ASIMOV risk--cost tradeoffs. (a, b) Targets of 5\%, 10\%, 15\%, and 20\%, with 20\% static preparation. (c, d) 10\%, 20\%, 30\%, and 40\% preparation at the 5\% target with inferred confidence; the dotted horizontal line is the nominal ensemble-risk target.}
    \label{fig:apx-asimov-sweeps}
\end{figure}

\textbf{Adaptive learning.}
\Cref{fig:apx-asimov-learning} presents adaptive learning through model improvement, escalation cost, and final-ensemble error.
After 500 samples, GP-based \method queries 47.4\% and has 21.4\% rolling model risk, compared with 65.6\% and 23.1\% for \method with PPI.
Both rectifiers improve with feedback and reach similar model risk by the end of the stream (18.7\% for GP and 18.5\% for PPI).
Learning useful corrections earlier lowers GP \method's average query rate over the full stream to 50.8\%, compared with 57.9\% for \method with PPI.

\begin{figure}[th]
    \centering
    \begin{subfigure}[t]{0.32\linewidth}
        \centering
        \includegraphics[width=\linewidth]{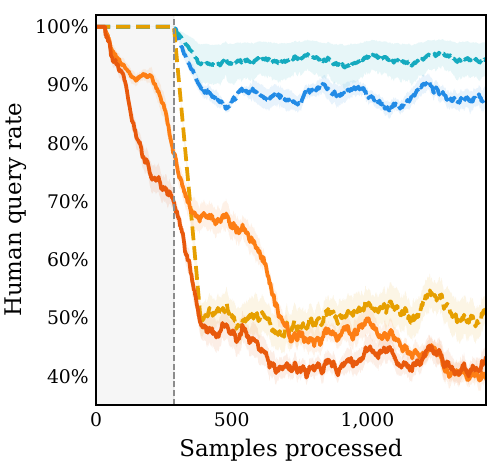}
    \end{subfigure}\hfill
    \begin{subfigure}[t]{0.32\linewidth}
        \centering
        \includegraphics[width=\linewidth]{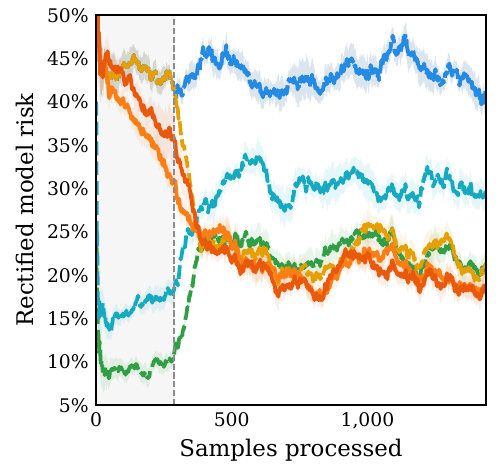}
    \end{subfigure}\hfill
    \begin{subfigure}[t]{0.32\linewidth}
        \centering
        \includegraphics[width=\linewidth]{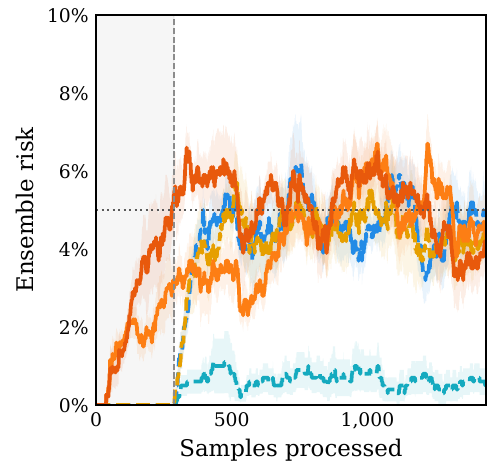}
    \end{subfigure}
    \includegraphics[width=\linewidth]{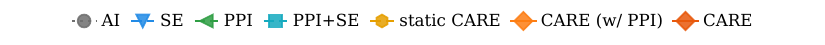}
    \caption{ASIMOV learning trajectories at the 5\% target and 20\% static preparation with inferred confidence. Lines average trailing windows of up to 100 arrivals over ten trials. Gray shading marks the 287-arrival static preparation period, where static model risks can include in-sample predictions. AI and PPI are shown as uncalibrated references in the model-risk panel. The dotted horizontal line is the nominal ensemble-risk target. Some initial model-risk transients are outside the plotted range.}
    \label{fig:apx-asimov-learning}
\end{figure}

\textbf{Robustness to confidence choice.}
At the 5\% target, replacing inferred with direct confidence changes \method's query rate by only $-1.5$ percentage points (50.8\% to 49.3\%), with mean ensemble risks of 4.76\% and 4.88\%, respectively.
For \method with PPI, the same change increases queries by 28.3 points (57.9\% to 86.2\%); SE changes by 2.2 points at 20\% preparation.
As in iSafetyBench, rectifying confidence together with predictions reduces reliance on the original confidence measure.

\FloatBarrier

\section{Risk decomposition}\label{apx:decomp}

\begin{repproposition}{prop:decomp}[Risk decomposition]
	Let $p_H(x)=P(Y=1\given X=x,H)$ be the human-aligned conditional probability, induced by human preference $H$.
	For a random training dataset $D$ drawn independently (conditioned on $H$) from the same distribution as $(X,Y)$, namely $P_{X}P_{Y|X,H}$, where $P_{Y|X,H}= \operatorname{Bern}(p_{H}(X))$,
	let $p_{D}=\mathbb{E}_{H}[p_{H}\given D]$ be the training-conditional learned predictor, and let $p_B=\EE_{D}[p_{D}]$ be the training-marginal counterpart.
	Then, we have
	\[
		&R_{2,x}(p_{A}) = \EE_{Y}[(p_A(X)-Y)^2\given X=x] \\
		=& \underbrace{\mathbb{E}_{H}[p_{H}(x)(1-p_{H}(x))]}_{\ds\text{aleatoric}}
		+ \underbrace{\mathbb{E}_{D,H}[(p_{H}(x)-p_{D}(x))^{2}]}_{\ds\text{excess}}
		+ \underbrace{\mathbb{E}_{D}[(p_{D}(x)-p_{B}(x))^{2}]}_{\ds\text{epistemic}}
		+ \underbrace{(p_{B}(x)-p_{A}(x))^{2}}_{\ds\text{bias}}
		.\]
\end{repproposition}
\begin{proof}
	By the tower property,
	\[
		&R_{2,x}(p_{A}) = \EE_{Y}[(p_A(X)-Y)^2\given X=x] \\
		=& \mathbb{E}_{H}[\EE_{Y}[(Y-p_{H}(X)+p_{H}(X)-p_{A}(x))^2\given X=x, H]] \\
		=& \mathbb{E}_{H}[p_{H}(x)(1-p_{H}(x))
				+ (p_{H}(x)-p_{A}(x))^{2}] \label{eq:decomp-1} \\
		=& \mathbb{E}_{H}[p_{H}(x)(1-p_{H}(x))]
		+ \mathbb{E}_{D}[ \mathbb{E}_{H} [(p_{H}(x)-p_{D}(x)
						+ p_{D}(x)-p_{A}(x))^{2}\given D] ] \\
		=& \mathbb{E}_{H}[p_{H}(x)(1-p_{H}(x))]
		+ \mathbb{E}_{D,H}[ (p_{H}(x)-p_{D}(x))^{2}
				+ (p_{D}(x)-p_{A}(x))^{2} ]\label{eq:decomp-2} \\
		=& \mathbb{E}_{H}[p_{H}(x)(1-p_{H}(x))]
		+ \mathbb{E}_{D,H}[(p_{H}(x)-p_{D}(x))^{2}]
		+ \mathbb{E}_{D}[(p_{D}(x)-p_{B}(x)
				+ p_{B}(x)-p_{A}(x))^{2}] \\
		=& \mathbb{E}_{H}[p_{H}(x)(1-p_{H}(x))]
		+ \mathbb{E}_{D,H}[(p_{H}(x)-p_{D}(x))^{2}]
		+ \mathbb{E}_{D}[(p_{D}(x)-p_{B}(x))^{2}]
		+ (p_{B}(x)-p_{A}(x))^{2} \label{eq:decomp-3},
	\]
	where the cross terms in
	\cref{eq:decomp-1,eq:decomp-2,eq:decomp-3} vanish, respectively, because
	$0=p_H(x)-\mathbb{E}_Y[Y\given X=x,H]
		=p_D(x)-\mathbb{E}_H[p_H(x)\given D]
		=p_B(x)-\mathbb{E}_D[p_D(x)]$.
\end{proof}

\begin{example}[Linear regression]\label{ex:linear}
	Suppose $p_{H}(x) = x^{T}(\beta + \epsilon_{H}(x))$, where $\epsilon_{H}(x)$ has zero mean and covariance $\Sigma_{x}$, and the AI model is null $p_{A}\equiv 0$. Consider perfect training-conditional and training-marginal learners recovering $\beta_{D}(x)\coloneqq \mathbb{E}_{H}[\beta + \epsilon_{H}(x)\given D]$ and $\beta_{B} \coloneqq \mathbb{E}_{D}[\beta_{D}] = \mathbb{E}_{H}[\beta + \epsilon_{H}(x)] = \beta$. Let $p_{D}(x) = x^{T}\beta_{D}(x)$ and $p_{B}(x) = x^{T}\beta_{B}$ be their corresponding predictors.
	By the law of total variance,
	\[
		&\mathbb{E}[(p_{H}(x)-p_{B}(x))^{2}] = x^{T}\Var(\epsilon_{H})x
		= x^T\mathbb{E}[\Var(\epsilon_H(x)\given D)]x + x^{T}\Var(\mathbb{E}[\epsilon_H(x)\given D])x \\
		=& \mathbb{E}[(p_{H}(x) - p_{D}(x))^{2}] + x^{T}\Var(\mathbb{E}[\epsilon_H(x)\given D])x
		\ge \mathbb{E}[(p_{H}(x) - p_{D}(x))^{2}]
		.\]
	When $\epsilon_H$ is homoskedastic and independent, $\mathbb{E}[\epsilon_H(x) \given D] = 0$, and thus the two predictors have the same Brier risk; when it's not, the training-conditional predictor $p_{D}$ can be strictly better by leveraging the richer information set $D$ to adapt to the local uncertainty structure.
\end{example}

\section{Extended discussion of rectification}
\label{apx:rect}

\subsection{Prior regression}

At an offline checkpoint, write the available labeled data as $D=\{(X_i,Y_i)\}_{i=1}^n$, let $\mathbf X$ have rows $X_i^T$, and let $\mathbf r=(r(X_1),\ldots,r(X_n))^T$. The prior in \cref{sec:rect} gives
\[
	\mathbf r\sim\Normal(\mathbf X\beta,K),\qquad
	K=\tau^2\Sigma,\qquad
	\Sigma=\rho K_\phi+(1-\rho)K_A,
\]
where $(K_\phi)_{ij}=k_\phi(X_i,X_j)$ and $(K_A)_{ij}=k_A(X_i,X_j)$, with $\tau>0$ and $\rho\in[0,1]$.

Conditional on the residual process, the labels are independent. By link symmetry, their log-likelihood contributions are $\ell_i(r)=\log g^{-1}\!\left((2Y_i-1)(g(p_A(X_i))+r)\right).$ We fit $(\beta,\phi,\tau,\rho)$ by maximizing the marginal likelihood
\[
	\Lambda(\beta,\phi,\tau,\rho)
	=\int_{\R^n}\exp\!\left(\sum_{i=1}^n\ell_i(r_i)\right)
	\Normal(\mathbf r;\mathbf X\beta,K)\,\mathrm d\mathbf r.
\]
The resulting posterior is
\[
	p(\mathbf r\given D)
	=\Lambda^{-1}\exp\!\left(\sum_{i=1}^n\ell_i(r_i)\right)
	\Normal(\mathbf r;\mathbf X\beta,K).
	\label{eq:gpc-exact-posterior}
\]
Let $\mathbf m$ and $V$ denote its posterior mean vector and covariance matrix, and set
$V_\beta=V+(\mathbf m-\mathbf X\beta)(\mathbf m-\mathbf X\beta)^T$.
Differentiating the marginal likelihood under the integral gives
\[
	\nabla_\beta\log \Lambda
	&=\mathbf X^T K^{-1}(\mathbf m-\mathbf X\beta), \\
	\nabla_{\alpha} \log \Lambda
	&=\frac12\tr\!\left((K^{-1}V_\beta K^{-1}-K^{-1})\nabla_{\alpha} K
	\right),\qquad \alpha\in\{\tau^2,\rho,\phi_j\}.
\]
In particular, setting the mean and amplitude derivatives to zero yields
\[
	\beta=(\mathbf X^T\Sigma^{-1}\mathbf X)^{-1}\mathbf X^T\Sigma^{-1}\mathbf m, \qquad
	\tau^2=\frac1n\tr(\Sigma^{-1}V_\beta)
	\label{eq:gpc-prior-updates}
	.\]

System \cref{eq:gpc-prior-updates} consists of coupled stationarity equations, as the posterior moments depend on the prior parameters.
This gives a practical expectation--maximization (EM) procedure. The E-step computes $(\mathbf m,V)$ under the current prior. Holding these moments fixed, the M-step increases the expected Gaussian log prior $-\frac12(\log|K|+\tr(K^{-1}V_\beta))$. Fixing $(\phi,\rho)$, update $\beta$ and then $\tau^2$ by \cref{eq:gpc-prior-updates}, recomputing $V_\beta$ after updating $\beta$. Then update $\rho\in[0,1]$ and $\phi$ numerically using $\nabla_{\rho} K=\tau^2(K_\phi-K_A)$ and $\nabla_{\phi} K=\tau^2\rho\nabla_{\phi} K_\phi$.

\subsection{Posterior inference}

Fix the fitted prior, abbreviate its kernel as $k=k_{\tau,\rho,\phi}$, and write $k(x,\mathbf X)=(k(x,X_1),\ldots,k(x,X_n))$. Gaussian conditioning gives
\[
	r(x)\given\mathbf r\sim\Normal\!\Bigl(
	x^T\beta+k(x,\mathbf X)K^{-1}(\mathbf r-\mathbf X\beta),
	k(x,x)-k(x,\mathbf X)K^{-1}k(\mathbf X,x)\Bigr).
\]
Integrating over \cref{eq:gpc-exact-posterior} gives the (non-Gaussian) posterior process, whose mean and covariance follow from total expectation and covariance:
\[
	m_t(x)
	&=x^T\beta+k(x,\mathbf X)K^{-1}(\mathbf m-\mathbf X\beta),
	\label{eq:gpc-exact-test-mean}\\
	k_t(x,x')
	&=k(x,x')-k(x,\mathbf X)K^{-1}k(\mathbf X,x')+k(x,\mathbf X)K^{-1}VK^{-1}k(\mathbf X,x').
	\label{eq:gpc-exact-test-variance}
\]
On the labeled inputs, $\mathbf m=m_t(\mathbf X)$ and $V=k_t(\mathbf X,\mathbf X)$; $K$ remains the prior covariance matrix. The formulas are exact when $(\mathbf m,V)$ are exact posterior moments. For implementation, we approximate the latent posterior by $\Normal(\mathbf m,V)$ using EP or Laplace moments below; the same formulas then define the approximate process $R_t\sim\GP(m_t,k_t)$.

Under this Gaussian approximation, the rectified prediction is the posterior mean:
\[
	p_t(x)=\int_\R g^{-1}(g(p_A(x))+r)
	\Normal(r;m_t(x),k_t(x,x))\,\d r.
	\label{eq:gpc-general-prediction}
\]
For a general link this is a one-dimensional integral. For the probit link $g=\Phi^{-1}$, introduce an independent $Z\sim\Normal(0,1)$.
We have $Z-R_t(x)\sim\Normal(-m_t(x),1+k_t(x,x))$ and thus
\[
	p_t(x)=\PP\!\left(Z-R_t(x)\leq g(p_A(x))\right)
	=\Phi\!\left(\frac{g(p_A(x))+m_t(x)}{\sqrt{1+k_t(x,x)}}\right),
	\label{eq:probit-predictive}
\]

\subsection{EP and ADF approximation}

\textbf{Offline: expectation propagation.}
EP approximates each likelihood factor by a Gaussian site \citep{minka2013ExpectationPropagation}:
\[
	\exp(\ell_i(r_i))\approx\widetilde t_i(r_i)
	\propto\exp\!\left(-\tfrac12\widetilde\lambda_i r_i^2+\widetilde\nu_i r_i\right).
\]
Multiplying the sites by the prior gives a Gaussian approximation with
\[
	V^{-1}=K^{-1}+\diag(\widetilde{\boldsymbol\lambda}),\qquad
	\mathbf m=V(K^{-1}\mathbf X\beta+\widetilde{\boldsymbol\nu}).
\]
To update site $i$, remove it from the current marginal to form the Gaussian cavity $\Normal(b,v)$, where
\[
	v=(V_{ii}^{-1}-\widetilde\lambda_i)^{-1},\qquad
	b=v(m_i/V_{ii}-\widetilde\nu_i).
\]
Multiply this cavity by $\exp(\ell_i(r))$ and compute the normalized tilted mean $\tilde b$ and variance $\tilde v$. Moment matching gives the new site parameters
\[
	\widetilde\lambda_i=\tilde v^{-1}-v^{-1},\qquad
	\widetilde\nu_i=\tilde b/\tilde v-b/v.
\]
Recompute the Gaussian approximation and revisit sites until the updates stabilize.

\textbf{Online: assumed-density filtering.}
ADF processes each new active label once.
Let the current approximate residual process be
$R_t \sim \GP(m_t,k_t)$.%
\footnote{As will be clear from the context, $R_t$ denotes both the marginal rectifier for the next (test) input $X$ and the approximate residual process on the entire input space $\X$.}
After an escalated query $(X_t,Y_t)$, set $b=m_t(X_t)$, $v=k_t(X_t,X_t)$, form the one-dimensional tilted marginal $\tilde{R} \propto g^{-1} ((2Y_{t}-1)(g(p_{A}(X_t))+\tilde{R}))\Normal(\tilde{R};b,v)$, and calculate $\tilde{b}=\EE[\tilde{R}]$ and $\tilde{v}=\Var(\tilde{R})$.
For $v>0$, Gaussian conditioning gives
\[
	\EE[R_t(x)\given R_t(X_t)=r]
	&=m_t(x)+v^{-1} k_t(x,X_t)(r-b),\label{eq:adf-conditional-mean}\\
	\Cov(R_t(x),R_t(x')\given R_t(X_t)=r)
	&=k_t(x,x')-v^{-1} k_t(x,X_t)k_t(X_t,x').
	\label{eq:adf-conditional-covariance}
\]
ADF matches the first two moments of the Gaussian approximation of $R_{t+1}$ to the above conditional moments using the laws of total expectation and covariance, giving the functional update:
\[
	m_{t+1}(x)
	&=m_t(x)+v^{-1} k_t(x,X_t)(\tilde{b} - b),\label{eq:adf-mean}\\
	k_{t+1}(x,x')
	&=k_t(x,x')+v^{-2}k_t(x,X_t)k_t(X_t,x')(\tilde v-v).
	\label{eq:adf-covariance}
\]
If $v=0$, positive semidefiniteness implies $k_t(x,X_t)=0$ for every $x$, and the process is unchanged.

Each update requires only a one-dimensional integral for the tilted marginal and a rank-one covariance correction.

\textbf{Closed-form probit updates.}
Differentiating the Gaussian density in tilted marginal with respect to $b$ yields
\begin{equation}
	\tilde b=b+v\,\partial_b\log N_t,\qquad
	\tilde v=v+v^2\,\partial_b^2\log N_t.
	\label{eq:adf-moment-identities}
\end{equation}
Indeed, $\partial_b\log\Normal(r;b,v)=(r-b)/v$, so $\partial_b\log N_t=(\tilde b-b)/v$ and a second differentiation gives $\partial_b^2\log N_t=(\tilde v-v)/v^2$.
For probit link $g=\Phi^{-1}$, let
\[
	\zeta=\frac{(2Y_t-1)(g(p_A(X_t))+b)}{\sqrt{1+v}},\qquad
	a(u)=\frac{\varphi(u)}{\Phi(u)},
	\label{eq:probit-tilt-notation}
\]
where $\varphi$ is the standard normal density. The same Gaussian convolution as in \cref{eq:probit-predictive} gives $N_t=\Phi(\zeta)$. Since $a'(u)=-a(u)(u+a(u))$, \cref{eq:adf-moment-identities} gives
\[
	\tilde b=b+\frac{(2Y_t-1)v}{\sqrt{1+v}}a(\zeta),\qquad
	\tilde v=v-\frac{v^2}{1+v}a(\zeta)(\zeta+a(\zeta)).
\]
Substitution into \cref{eq:adf-mean,eq:adf-covariance} gives the explicit functional updates
\[
	m_{t+1}(x) =&m_t(x)+\frac{(2Y_t-1)k_t(x,X_t)}{\sqrt{1+v}}a(\zeta),\\
	k_{t+1}(x,x') =&k_t(x,x')-\frac{k_t(x,X_t)k_t(X_t,x')}{1+v}a(\zeta)(\zeta+a(\zeta)).
\]

\subsection{Laplace and MAP approximation}

\textbf{Offline: Laplace approximation.}
Laplace approximation uses a mode and local curvature rather than matching posterior moments to determine the Gaussian approximation.
Specifically, it replaces the log posterior by its quadratic expansion at the maximum a posteriori (MAP) residual vector \citep{rasmussen2008GaussianProcesses}:
\[
	\widehat{\mathbf r}^{\mathrm{MAP}}
	=\argmax_{\mathbf r}\left\{\sum_{i=1}^n\ell_i(r_i)
	-\frac12(\mathbf r-\mathbf X\beta)^TK^{-1}(\mathbf r-\mathbf X\beta)\right\}.
\]
The gradient is $(\ell_i'(r_i))_{i=1}^n-K^{-1}(\mathbf r-\mathbf X\beta)$, and the negative Hessian is $K^{-1}+\diag(-\ell_i''(r_i))$, giving Newton iterations for the mode.
At the mode, the Gaussian approximation has
\[
	\mathbf m=\widehat{\mathbf r}^{\mathrm{MAP}},\qquad
	V=(K^{-1}+\diag(-\ell_i''(\widehat r_i^{\mathrm{MAP}})))^{-1}.
\]
Substituting these into \cref{eq:gpc-exact-test-mean,eq:gpc-exact-test-variance} gives the approximate posterior process. Predictions still integrate over this Gaussian process as in \cref{eq:gpc-general-prediction}.

\textbf{Online: Newton step for maximum a posteriori.}
The Gaussian log-density $\log\Normal(r;b,v)$ is already quadratic. Expand the log-likelihood at $r=b$:
\[
	\ell_t(r)\approx\ell_t(b)+\ell_t'(b)(r-b)+\tfrac12 \ell_t''(b)(r-b)^2
\]
Set $d=\ell_t'(b)$ and $w=-\ell_t''(b)$. The local expansion
gives a Gaussian tilted approximation with mean $b+vd/(1+vw)$ and variance $v/(1+vw)$. Substituting these for $(\tilde b,\tilde v)$ in \cref{eq:adf-mean,eq:adf-covariance} gives
\[
	m_{t+1}(x)
	&=m_t(x)+\frac{k_t(x,X_t)d}{1+vw},\label{eq:newton-mean}\\
	k_{t+1}(x,x')
	&=k_t(x,x')-\frac{w\,k_t(x,X_t)k_t(X_t,x')}{1+vw}.
	\label{eq:newton-covariance}
\]

\textbf{Closed-form logistic and probit updates.}
For the logit link, let $p=g^{-1}(g(p_A(X_t))+b)$. Direct differentiation gives $d=Y_t-p$ and $w=p(1-p),$ which yields the usual recursive logistic Newton update through \cref{eq:newton-mean,eq:newton-covariance}.
For the probit link, use the ratio $a$ from \cref{eq:probit-tilt-notation} and set $\zeta=(2Y_t-1)(g(p_A(X_t))+b)$. Differentiating $\ell_t(b)=\log\Phi(\zeta)$ gives $d=(2Y_t-1)\,a(\zeta)$ and $w=a(\zeta)(\zeta+a(\zeta))$. Here $\zeta$ has no factor $\sqrt{1+v}$ because we evaluate the likelihood derivatives at $b$.

\section{Extended analysis of calibration}\label{apx:cali}

\textbf{Initial: confidence-based nonconformity score.}
It's natural to use the confidence measure $c_t$ (provided by the AI model and rectified by the rectifier) to construct our consistency check, referred to as a \emph{nonconformity score}.
Formally, for a (rectified) model consisting of a classifier and a confidence function $\M_{t} = (f_t,c_t)$,
its \emph{decision set} is given by \cref{eq:care}:
\[
	\mathcal{S}(\lambda; \M_t, x) \coloneqq \begin{cases}
		\{f_{t}(x)\}, & \text{if } c_t(x) > \lambda(x); \\
		\mathcal{Y},  & \text{o.w.}
	\end{cases}
\]
That is, when the decision set is a singleton, the model is confident enough to make a decision; otherwise, it is ambiguous and requires human review.
The corresponding nonconformity score is the generalized inverse of this decision set:
\[
	s(x,y;\M_t) \coloneqq \inf\{\lambda\in[0,1]:y\in \mathcal{S}(\lambda; \M_t,x)\}
	=c_t(x)\Ind\{f_t(x)\neq y\},
\]
which is zero if $y$ agrees with the prediction and equals the model confidence when it disagrees.
Intuitively, when $y$ is the true label, a confident error receives a higher score than an uncertain one.
For $c_t(x)=\max\{p_t(x),1-p_t(x)\}$ in binary classification, this recovers the prediction-based score of \citet{romano2020ClassificationValid}.
Our construction is tailored to our escalation rule and accommodates model-provided confidence measures.

\subsection*{Step 1. Augmented particles}

We first reformulate our method more concisely as follows. Fix $t_0\le t$, with the last offline checkpoint preceding arrival $X_{t_0}$.
Let $D_{0}$ collect data available at that checkpoint and $D_{t_0:t-1}$ collect actively escalated data from the online stream at times $t_0,\ldots,t-1$.
For each particle, let $D^{(k)}$ be the corresponding left-out data set.
For simplicity, we suppose all left-out data sets have the same size $m$.
Let $D^{(K+1)}$ augment the incoming test input $X_t$ with the test label $y$ plus another $m-1$ data points in $D_0$ to match the size of the left-out data sets.%
\footnote{%
	A simple mental model is that $\{ D^{(k)}  \}_{k=1}^{K+1}$ are disjoint. For example, for $K$-fold cross-conformal, we can first reserve $(n_0+1)/(K+1)-1$ data points for test fold augmentation, and then partition the remaining data points into $K$ disjoint sets. Leave-one-out cross-conformal is another example where no test augmentation is needed.
	However, all of our analysis goes through without requiring the disjointness, via the multiset formalism.%
}
These left-out datasets should be chosen by a fixed rule independent of the data values.
When $m=1$, we recover the leave-one-out setting.
Let $D^{-(k,k')}\coloneqq (D_{0}\cup D^{(K+1)}) \setminus (D^{(k)}\cup D^{(k')})$, and define augmented model particles as
\[
	\M^{-(k,k')} = \M_\mathsf{on}\left(\mathsf{M}\left(D^{-(k,k')}\right), D_{t_0:t-1}\right),\quad k,k'\in[K+1]
	.\]
The $k$-th actual model particle is then $\M^{(k)} = \M^{-(k,K+1)}$.
For each left-out dataset, we designate a representative data point $z^{(k)}$ by a fixed slot rule, and designate $z^{(K+1)}\coloneqq(X_t,y)$ for $D^{(K+1)}$.
We associate each particle with the nonconformity scores evaluated on their designated data points:
\[
	s^{-(k,k')} = s(\M^{-(k,k')}, z^{(k)}),\quad k\ne k'\in[K+1]
	,\]
or $s^{-(k,k')}=\infty$ when $k=k'$.
Note that the actual nonconformity score of the $k$-th particle at the test input is $s(\M^{(k)},(X_t,y)) = s^{-(K+1,k)}$.

\subsection*{Step 2. Weighted conformity}

When all left-out datasets are exchangeable and the learning method is permutation-invariant, the scores $s^{-(k,k')}$ are equivariant and thus conformal with the test score $s^{-(K+1,k)}$.
However, in our setting, the left-out datasets are drawn from different distributions due to the adaptive data stream, and thus they have different impact on the model particles.
We thus need to weigh their conformity w.r.t. the test score.

To do so, we need to delineate two things each data point specifies: a time slot and a realized sample value.
Then, we consider the following conditioning: $K+1$ groups of left-out time slots are realized, but the mapping from each slot group to the sample group is not revealed.
Under this conditioning, each sample group has some probability of occupying the test slot group, and thus the scores should be weighed by the corresponding probability.

Formally, define the following conditioning information $\mathcal{D}$:
the full redacted arrival slot log which records the escalation decisions is known;
the grouped slots of all left-out datasets are known;
all collected samples are observed;
the slot-to-sample mapping of slots outside the union of the left-out datasets, including the online slots, is known;
and the grouping and within-group order of the left-out samples, including the test point, are known.
The only thing we mask---the mapping from left-out slot groups to left-out sample groups---is then an element of the $(K+1)$-permutation group $S_{K+1}$.
We can further augment $\mathcal{D}$ by restricting the mapping to be in a transitive subgroup $G$ of $S_{K+1}$, e.g., the $(K+1)$-cyclic group.
Thus, the remaining assignments are exactly those obtained by applying $G$ to the realized assignment.
Our results hold for any permutation subgroup.
A smaller group reduces the number of permutations to consider, hence making likelihood calculation easier; however, it may produce more concentrated weights and reduce the diversity of our particle system.
For example, if the subgroup is non-transitive and never maps a fold to the test slot, then the corresponding particle has zero weight and will not contribute.

Let $D^{\circ}_{t,z}$ be the observed data stream $((X_i,A_i,A_iY_i))_{i=1}^{t-1}$ appended with the test point $z=(X_t,y)$, where unqueried labels are redacted.
With a slight abuse of notation, we also denote $G$ as the extended group of permutations on the entire data stream, which is the identity on the slots outside the union of the left-out datasets, including the unqueried slots and online stream, and keeps the relative order of the slots inside each left-out dataset.
An immediate result is that the permutations in $G$ do not change our collection of model particles. Specifically, for any $\sigma\in G$ applied on the entire data stream, the new model particles are
\[
	&\M^{-(k,k')}(\sigma(D^{\circ}_{t,z})) \\
	&= \M_\mathsf{on}\left(\mathsf{M}\left(\sigma((D_0\cup\{z\})\setminus (D^{(\sigma^{-1}(k))}\cup D^{(\sigma^{-1}(k'))}))\right), \sigma(D_{t_0:t-1})\right) \label{eq:permu-1}\\
	&=\M_\mathsf{on}\left(\mathsf{M}\left(\sigma(D^{-(\sigma^{-1} (k),\sigma^{-1} (k'))})\right), D_{t_0:t-1}\right)\label{eq:permu-2}\\
	&=\M_\mathsf{on}\left(\mathsf{M}\left(D^{-(\sigma^{-1} (k),\sigma^{-1} (k'))}\right), D_{t_0:t-1}\right)\label{eq:permu-3}\\
	&=\M^{-(\sigma^{-1} (k),\sigma^{-1} (k'))}(D^{\circ}_{t,z})
	,\quad k,k'\in[K+1]
	,\]
where \cref{eq:permu-1} is because $\sigma$ is separable between the online stream and the rest of the data stream, and when we leave out the $k$-th batch of slots under sample permutation $\sigma$, we are actually removing the original $\sigma^{-1}(k)$-th left-out dataset;
\cref{eq:permu-2} is because $D_0\cup D^{(K+1)}=D_0\cup\{z\}$ and $\sigma$ preserves the online stream;
and \cref{eq:permu-3} is because $\mathsf{M}$ is permutation-invariant.

As a result, under permutation $\sigma$, the original score $s^{-(k,k')}$ moves to position $(\sigma(k),\sigma(k'))$, for any $k,k'\in[K+1]$, and thus the score array, indexed by sample groups, is fixed under $\mathcal{D}$.%
\footnote{We also require that elements of $G$ map designated data points to designated data points.}
Therefore, the probability that the $k$-th sample group supplies the test-score row is exactly its probability of occupying the test slot group under $\mathcal{D}$.
This need not be true when the learning method is permutation-variant w.r.t. elements in $G$, as the scores may change under each permutation.
Our method resolves this by restricting the permutation group to a subgroup $G$ under which the learning method is permutation-invariant.

\subsection*{Step 3. Risk control}

Let the probability of the $k$-th sample group occupying the test slot under $\mathcal{D}$ be $w^{(k)}$; note that $w^{(k)}$ depends on $y$ and is a functional of $X_t$, but here we suppress this dependence for notational simplicity.
Then, the probability that group $k$ wins its comparison with the random test group is
\[
	& \sum_{k'=1}^{K+1} w^{(k')} \Ind\left\{s(\M^{-(k,k')},z^{(k)} ) > s(\M^{-(k,k')},z^{(k')} ) \right\} \\
	=& \sum_{k'=1}^{K+1} w^{(k')} \Ind\{s^{-(k,k')} > s^{-(k',k)} \}
	.\]
We see that each summand puts $k'$ at the test slot and is scaled with its likelihood.
Following \citet{barber2021PredictiveInference}, we call a group \emph{strange} if the above probability is unusually large, giving the \emph{strange set}:
\[
	\mathcal{U}\coloneqq \left\{ k\in [K+1]: \sum_{k'=1}^{K+1}w^{(k')}\Ind \left\{ s^{-(k,k')} > s^{-(k',k)} \right\} \ge 1-\epsilon /2 \right\}
	.\]
Recall that each score has some probability to be the test score. Thus, the probability of the test group being strange is $u=\sum_{k=1}^{K+1} w^{(k)} \Ind\left\{ k\in \mathcal{U} \right\}$. We are left to show that this probability is small.
By definition, we have
\[
	u(1-\epsilon /2) =& \sum_{k=1}^{K+1}w^{(k)}\Ind \left\{ k\in \mathcal{U} \right\} \cdot (1-\epsilon /2) \\
	\le& \sum_{k=1}^{K+1}w^{(k)}\Ind \left\{ k\in \mathcal{U} \right\} \sum_{k'=1}^{K+1} w^{(k')} \Ind\{s^{-(k,k')} > s^{-(k',k)}\} \\
	=& \sum_{k\in \mathcal{U},k'\in[K+1]} w^{(k)} w^{(k')}  \Ind\{s^{-(k,k')} > s^{-(k',k)}\} \\
	=& \left( \sum_{k\in \mathcal{U},k'\in \mathcal{U}} + \sum_{k\in \mathcal{U}, k'\in[K+1]\setminus \mathcal{U}} \right) w^{(k)} w^{(k')}  \Ind\{s^{-(k,k')} > s^{-(k',k)}\}
	.\]
For the first summation, we have
\[
	&\sum_{k\in \mathcal{U},k'\in \mathcal{U}} w^{(k)} w^{(k')}  \Ind\{s^{-(k,k')} > s^{-(k',k)}\}\\
	=& \frac{1}{2} \sum_{k\in \mathcal{U},k'\in \mathcal{U}} w^{(k)} w^{(k')}  (\Ind\{s^{-(k,k')} > s^{-(k',k)}\} + \Ind\{s^{-(k,k')} < s^{-(k',k)}\})\\
	\le& \frac{1}{2} \sum_{k\in \mathcal{U},k'\in \mathcal{U}} w^{(k)} w^{(k')}  \left( 1 - \Ind \left\{k=k'\right\} \right)\\
	=& \frac{1}{2} \left(u^{2} - \sum_{k\in \mathcal{U}} (w^{(k)})^{2}\right)
	,\]
where the inequality is because there are cases when $s^{-(k,k')} = s^{-(k',k)}$ but $k\ne k'$.
The second summation can be bounded by
\[
	\sum_{k\in \mathcal{U},k'\not\in \mathcal{U}} w^{(k)} w^{(k')}  \Ind\{s^{-(k,k')} > s^{-(k',k)}\}
	\le\sum_{k\in \mathcal{U},k'\not\in \mathcal{U}} w^{(k)} w^{(k')}
	=\sum_{k\in \mathcal{U}} w^{(k)} (1-u)
	=u (1-u)
	.\]
Combining both gives
\[
	u(1-\epsilon /2) \le \frac{1}{2}u^{2} + u(1-u) = u - \frac{1}{2}u^{2}
	\implies u \le \epsilon
	.\]

Connecting back to our prediction set
\[
	\calC_{t,\epsilon}(x)
	=\left\{y\in\Y:\sum_{k=1}^{K}w^{(k)}(x,y)\Ind\left\{s(\M^{(k)},z^{(k)})<s(\M^{(k)},(x,y))\right\}<1-\epsilon/2\right\}
\]
gives
\[
	P(Y_t\notin \calC_{t,\epsilon}(X_t)) = \EE[u] \le \epsilon
	.\]

\subsection*{Step 4. Counterfactual likelihood}

We are left to calculate the weight functional $w^{(k)}$ under $\mathcal{D}$.
Writing $X_i^\sigma=X_{\sigma^{-1}(i)}$, we have
\[
	&w^{(k)}(x,y) = \frac{\sum_{\sigma\in G:\sigma(k)=K+1}\Lambda_t^\sigma(x,y)}{\sum_{\sigma\in G}\Lambda_t^\sigma(x,y)}, \label{eq:inv-cond}\\
	\text{where}\quad &\Lambda_t^\sigma(x,y) = \prod_{i=1}^{t-1}\ell(X_i^\sigma,A_i\given\mathcal{F}_{i-1}^\sigma),\\
	\text{where}\quad &\ell(x,a\given \mathcal{F}) = \pi(x\given\mathcal{F})^{a}(1-\pi(x\given\mathcal{F}))^{1-a}
	,\]
where $(\mathcal{F}_{i}^{\sigma})_{i=0}^{t-1}$ is the observed filtration generated by replaying the learning and acquisition algorithms on the counterfactual stream $\sigma(D^{\circ}_{t,z})$;
$\pi(x\given \mathcal{F})$ is the algorithm's probability to escalate $x$ given the filtration $\mathcal{F}$;%
\footnote{In this section, $A$ denotes the general escalation indicator, and $\pi$ its corresponding escalation probability.}
and \cref{eq:inv-cond} follows because the i.i.d. covariate and conditional-label density factors in the unconditioned stream likelihood are invariant under $G$ and cancel upon conditioning on $\mathcal{D}$. The product includes unqueried arrivals but no acquisition factor for the test point.

We see that to calculate the weight functionals, we have to replay the learning algorithm on each counterfactual stream $\sigma(D^{\circ}_{t,z})$ for all $\sigma\in G$.
For the online period, we have a simplifying factorization:
\[
	w^{(k)}(x,y) \propto& \sum_{\sigma\in G :\sigma(k)=K+1}  \Lambda^{\sigma}_{t_0}(x,y) \cdot \prod_{j=t_0}^{t-1} \ell(X_{j},A_{j}\given \mathcal{F}_{t_0-1}^{\sigma}, (Z_{j'})_{j'=t_0}^{j-1}) \\
	=& \left( \sum_{\sigma\in G :\sigma(k)=K+1} \Lambda^{\sigma}_{t_0}(x,y) \right) \cdot \prod_{j=t_0}^{t-1} \ell(X_{j},A_{j}\given \mathcal{F}^{(k)} _{t_0-1}, (Z_{j'})_{j'=t_0}^{j-1}) \\
	\propto & w^{(k)}_{t_0}(x,y) \cdot \prod_{j=t_0}^{t-1} \ell(X_{j},A_{j}\given \mathcal{F}^{(k)} _{t_0-1}, (Z_{j'})_{j'=t_0}^{j-1})
	,\]
where $\mathcal{F}^{(k)} _{t_0-1}$ denotes the counterfactual history obtained by switching only the designated data point $z^{(k)}$ in the $k$-th left-out dataset with the test point $z$ and no other permutation.
We use the fact that permutations in $\left\{ \sigma\in G : \sigma(k)=K+1 \right\}$ generate the same set of offline data points $(D_0\cup\{z\})\setminus\{z^{(k)}\}$, and that the online period is order-invariant with respect to the offline dataset.
Therefore, after obtaining the weights $w^{(k)}_{t_0}$ after an offline checkpoint, we only need to co-evolve the particles to maintain $\mathcal{F}^{(k)}_{t-1}$ without doing any replay, enabling efficient online adaptation of the weights.
Again, this does not hold if either the offline learning is order-variant or the permutation changes the online stream.

\section{Extended analysis of safety escalation} \label{apx:esc}

\subsection{Proof of \texorpdfstring{\cref{cor:cali}}{Corollary 1}}

The risk of our ensemble in \cref{eq:care} reads
\[
	R_{01}(\hat{Y}_t) =& P(\hat{Y}_t\ne Y_t)\\
	=& P((\hat{Y}_t\ne Y_t)\land(A_t \lor E_t)) + P((\hat{Y}_t\ne Y_t)\land \lnot(A_t \lor E_t))\\
	=& P(A_t\lor E_t)P(Y_t\ne Y_t\given A_t\lor E_t) + P((\hat{Y}_t\ne Y_t) \land \lnot A_t \land \lnot E_t) \\
	\le& P((\hat{Y}_t\ne Y_t) \land \lnot E_t) \\
	\le& P((1-f_t(X_t)= Y_t) \land (c_t(X_t) > \lambda_{t}(X_t))) \label{eq:exclusion}
	.\]
Therefore, we only need to bound the probability of the true label being the opposite label when we accept the current prediction.
By definition \cref{eq:lambda},
\[
	& c_t(X_t) > \lambda_{t}(X_t) \\
	\iff& c_t(X_t) - \gamma^{c}_t(X_t)  > \q_{\frac{1-\epsilon /2}{W_{t}(x)}}\left(\sum_{k\in \mathcal{W}_t(x)} \frac{w_{t}^{(k)} (x)}{W_{t}(x)}\delta\{s_{t}^{(k)} \}\right) \\
	\implies& \min_{k\in \mathcal{W}_t(X_t)}c_t^{(k)}(X_t) > \q_{\frac{1-\epsilon /2}{W_{t}(x)}}\left(\sum_{k\in \mathcal{W}_t(x)} \frac{w_{t}^{(k)} (x)}{W_{t}(x)}\delta\{s_{t}^{(k)} \}\right) \\
	\iff& \min_{k\in \mathcal{W}_t(X_t)}c_t^{(k)}(X_t) > \q_{1-\epsilon /2}\left( \sum_{k=1}^{K} w_{t}^{(k)} \delta\{\iota_{k}s_{t}^{(k)} \} \right)
	,\]
where we use shorthand $s_t^{(k)}\coloneqq s_t^{(k)}(z^{(k)})$ and $\iota_{k}$ is the characteristic function that equals $1$ if $k\in \mathcal{W}_{t}(X_t)$ and $+\infty$ otherwise.
Further, this condition is equivalent to
\[
	&\sum_{k\in \mathcal{W}_t(X_t)} w_{t}^{(k)} \Ind\left\{\min_{k'\in \mathcal{W}_t(X_t)}c_t^{(k')}(X_t) > s_{t}^{(k)}\right\} \ge 1-\epsilon /2\\
	\implies&\sum_{k\in \mathcal{W}_t(X_t)} w_{t}^{(k)} \Ind\left\{c_t^{(k)}(X_t) > s_{t}^{(k)}\right\} \ge 1-\epsilon /2
	.\]
Recall that $\{w^{(k)}_t\}$ concern the opposite label $1-f_t(X_t)$, and $f_t^{(k)} (X_t) = f_t(X_t)$ for all agreeing particles $k\in \mathcal{W}_t(X_t)$.
Thus, the above condition is equivalent to
\[
	&\sum_{k\in \mathcal{W}_t(X_t)} w_{t}^{(k)}(X_t,1-f_t(X_t)) \Ind\left\{c_t^{(k)}(X_t) > s_{t}^{(k)}(z^{(k)} )\right\} \ge 1-\epsilon /2 \\
	\iff&\sum_{k=1}^{K}  w_{t}^{(k)}(X_t,1-f_t(X_t)) \Ind\left\{c_t^{(k)}(X_t) \Ind \{ f_t^{(k)}(X_t) \ne 1-f_t(X_t) \} > s_{t}^{(k)}(z^{(k)} )\right\} \ge 1-\epsilon /2 \\
	\iff&\sum_{k=1}^{K}  w_{t}^{(k)}(X_t,1-f_t(X_t)) \Ind\left\{s_t^{(k)}(X_t, 1-f_t(X_t) ) > s_{t}^{(k)}(z^{(k)})\right\} \ge 1-\epsilon /2
	.\]
By definition \cref{eq:pred-set}, the above condition is equivalent to
\[
	1-f_t(X_t) \notin \calC_{t,\epsilon}(X_t)
	.\]
Plugging this indication back into \cref{eq:exclusion} gives
\[
	R_{01}(\hat{Y}) \le& P((1-f_t(X_t) = Y_t) \land (1-f_t(X_t) \notin \calC_{t,\epsilon}(X_t))) \\
	=& P((1-f_t(X_t) = Y_t) \land (Y_t \notin \calC_{t,\epsilon}(X_t))) \\
	\le& P(Y_t \notin \calC_{t,\epsilon}(X_t)) \\
	\le& \epsilon.
	,\]
where the final inequality follows from \cref{thm:cali}.

We remark that with equal weights, excluding the opposite label indicates the inclusion of the predicted label in the prediction set $\mathcal{C}_{t,\epsilon}$.
However, this is no longer true under covariate shift.
Nonetheless, \cref{eq:exclusion} bounded by $\epsilon$ implies that it is still safe to accept the prediction.

\subsection{Approximate weights and scores}

During the offline period, calculating \cref{eq:off-weight} can involve numerous policy replays, even with a simple subgroup $G$ as discussed in \cref{apx:cali}.
During the online period, although no counterfactual replays are needed, prioritizing computational efficiency may still call for approximation.
As long as we have a high-probability bound on the approximation error, safety guarantee in \cref{thm:cali} still holds.

Consider approximate weights $\{ \hat{w}^{(k)} \}_{k=1}^{K}$ and scores  $\{ \hat{s}^{(k)} \}_{k=1}^{K}$.
For notational simplicity, we omit the time index $t$ in this subsection.
Following \cref{eq:lambda}, we denote the \emph{exact} calibrated threshold using accurate weights and scores as
\[
	\lambda ^*(x) = \lambda(x; \{w ^{(k)}\}_{k=1}^{K}, \{s ^{(k)} \}_{k=1}^{K}) \coloneqq \q_{\frac{1-\epsilon}{W(x)}}\left( \sum_{k=1}^{K} \frac{w ^{(k)}(x)}{W(x)} \delta\{s ^{(k)} \} \right) + \gamma^{c}(x)
	.\]
Then, if we know with high probability $1-O(\epsilon)$ that the exact $(w,e)$ is in an uncertainty set $\Omega$, we can set
\[
	\lambda(x) = \sup_{(w,e)\in \Omega}\lambda(x; w,e),
\]
which guarantees an excess risk of $O(\epsilon)$.

We give an example with two simple approximation bounds.
Suppose that, for all $k\in[K]$, we have $s ^{(k)}  \le \hat{s}^{(k)}  + \gamma^{s}_{k}$ for scores and $\|w-\hat{w}\|_{\mathrm{TV}}\le \gamma^{w}$ for weights. Then, we can set
\[
	\lambda (x) = \q_{1-\epsilon /2 + \gamma^{w}} \left( \sum_{k=1}^{K}\hat{w}^{(k)}(x)\delta \{\hat{s}^{(k)}+\gamma^{s}_k\} \right) + \gamma^{c}(x)
	,\]
which satisfies \cref{cor:cali}.

\section{Extended discussion of active acquisition} \label{apx:ada}

We first remark that thanks to its modular design, \method can be plugged with any acquisition strategy in the active learning/inference literature \citep{zrnic2026ActiveStatistical}.
For simplicity, our experiments implement safe escalation only ($\pi_t(X_t) = \Ind\{c_t(x_t)\le \lambda_t(x_t)\}$).
Below, we discuss a more general active acquisition strategy tailored to our goal of minimizing query complexity.

Let $A_t$ and $E_t$ be the active and safe escalation decisions at time $t$, respectively.
Concerning short-term efficiency, $A_t$ should escalate a query that is likely to be escalated by $E_t$;
concerning long-term efficiency, $A_t$ should escalate a query that minimizes the future escalation frequency of $E\lor A$.
Ideally, active escalation optimizes
\[
	\min_{\pi_t\in[0,1]} \mathbb{E}_{A_t \sim \operatorname{Bern}(\pi_t)}\left[ A_t(1 - E_t) + \mathbb{E}\left[ \sum_{t'\ge t+1} E_{t'}\lor A_{t'} \middle| A_t \right] \right]
	,\]
where the first term is the additional cost of active escalation at this step, and the second term is the expected future escalation cost.
However, this objective is hard to evaluate as
\begin{enumerate*}[label=(\roman*)]
	\item the future is hard to estimate due to adaptation;
	\item $E_t$ is unknown as we need to avoid $\pi_t$ depending on $\lambda_t$; and
	\item the decision indicators are not smooth.
\end{enumerate*}
Therefore, we consider a myopic and approximate objective:
\[
	& \mathbb{E}_{A_t \sim \operatorname{Bern}(\pi_t)}\left[ A_t(1 - \hat{E}_t) + \eta\mathbb{E}\left[ \hat{E}_{t+1} \middle| A_t \right] \right] \\
	= & \mathbb{E}_{A_t \sim \operatorname{Bern}(\pi_t)}\left[ A_t(1 - \Ind \{ c_t(X_t) < 1-\epsilon \}) + \eta\mathbb{E}\left[ \Ind \{ \hat{c}_{t+1}(X_{t+1}) < 1-\epsilon \} \middle| A_t \right] \right] \\
	= & \pi_t\Ind \{ c_t(X_t) \ge 1-\epsilon \} + \eta (1-\pi_t) P_{X}(c_t(X) < 1-\epsilon) \\
	&+ \eta\pi_{t} \left( p_tP_{X} ( c_{t+1}^{\sss(\!X_t\!,1\!)}(X) < 1-\epsilon ) + \pi_t(1-p_t)P_{X} ( c_{t+1}^{\sss(\!X_t\!,0\!)}(X) < 1-\epsilon )\right)
	,\]
where $p_t = p_t(X_t)$ and $X_{t+1} \sim P_{X}$; besides the one-step lookahead scaled by the effective horizon parameter $\eta$, we also approximate the safe escalation threshold by $1-\epsilon$ and the label distribution of $Y_{t}$ by the model prediction $\operatorname{Bern}(p_{t})$.
Specifically, $\hat{E}_t(c)\coloneqq \Ind \{ c < 1-\epsilon \}$ and $\hat{c}_{t+1} = \mathbb{E}_{\hat{Y}_t \sim \operatorname{Bern}(p_t)}[c^{\sss(\!X_t\!,\hat{Y}_t\!)}_{t+1}]$ where $c^{\sss(\!X_t\!,\hat{Y}_t\!)}_{t+1}$ is the rectifier confidence at $t+1$ after updating the rectifier with the new query $(X_t,\hat{Y}_t)$.

The functional $c \mapsto P_{X}(c(X) < 1-\epsilon)$ may be hard to evaluate. Additionally, the above objective gives a hard decision $\pi_{t}\in\{ 0, 1 \}$. We thus further smooth the objective, giving a \emph{confidence-maximizing} active escalation rule
\[
	\pi_t = g^{-1}\circ h^{(1)}( p_t h^{(2)} _{\epsilon}(c_{t+1}^{\sss(\!X_t\!,1\!)} ) + (1-p_t) h^{(2)}_{\epsilon}(c_{t+1}^{\sss(\!X_t\!,0\!)} ) - h^{(2)}_{\epsilon}(c_t) - \eta^{-1}\Ind \{ c_t(X_t) \ge 1-\epsilon \})
	,\]
where both $h^{(1)}$ and $h^{(2)}_{\epsilon}$ are strictly increasing function, with $h^{(1)}$ being affine onto $\R$ and $h^{(2)} _{\epsilon}(c)$ approximates $P_{X}(c(X)\ge 1-\epsilon)$.

\end{document}